\documentclass[abstract]{scrartcl}  
\makeatletter
\DeclareOldFontCommand{\rm}{\normalfont\rmfamily}{\mathrm}
\DeclareOldFontCommand{\sf}{\normalfont\sffamily}{\mathsf}
\DeclareOldFontCommand{\tt}{\normalfont\ttfamily}{\mathtt}
\DeclareOldFontCommand{\bf}{\normalfont\bfseries}{\mathbf}
\DeclareOldFontCommand{\it}{\normalfont\itshape}{\mathit}
\DeclareOldFontCommand{\sl}{\normalfont\slshape}{\@nomath\sl}
\DeclareOldFontCommand{\sc}{\normalfont\scshape}{\@nomath\sc}
\makeatother
\usepackage{amsmath}
\usepackage{amssymb}
\usepackage{amsthm}
\usepackage[boxed,nosemicolon]{algorithm2e}
\usepackage{graphicx}
\usepackage[tight,hang,bf]{subfigure}
\usepackage{booktabs}
\usepackage{url}
\usepackage{rotating}
\usepackage{natbib}
\usepackage{caption}
\setcapindent{0cm}

\newcommand{\footremember}[2]{%
    \footnote{#2}
    \newcounter{#1}
    \setcounter{#1}{\value{footnote}}%
}
\newcommand{\footrecall}[1]{%
    \footnotemark[\value{#1}]%
} 

\newcommand{\pairedimagewidth}{0.47\columnwidth}
\newcommand{\smartqed}{}

\title{Identifying Consistent Statements about Numerical Data with Dispersion-Corrected Subgroup Discovery}

\author{Mario Boley\footremember{mmci}{Max Plank Institute for Informatics and Saarland University, Saarbr\"ucken, Germany}\footremember{fhi}{Fritz Haber Institute of the Max Planck Society, Berlin}\\
\texttt{mboley@mmci.uni-saarland.de}
\and
Bryan R.~Goldsmith\footrecall{fhi}\\
\texttt{goldsmith@fhi-berlin.mpg.de}
\and
Luca~M.~Ghiringhelli\footrecall{fhi}\\
\texttt{ghiringhelli@fhi-berlin.mpg.de}
\and
Jilles Vreeken\footrecall{mmci}\\
\texttt{jilles@mpi-inf.mpg.de}
}

\newcommand{\defemph}[1]{\textbf{#1}}
\newtheorem{theorem}{Theorem}
\newtheorem{prop}[theorem]{Proposition}
\newtheorem{lem}[theorem]{Lemma}
\theoremstyle{definition}
\newtheorem{definition}{Definition}

\newcommand{\cC}{\mathcal{C}}
\newcommand{\cE}{\mathcal{E}}
\newcommand{\cF}{\mathcal{F}}
\newcommand{\cL}{\mathcal{L}}

\newcommand{\cO}{\mathcal{O}}
\newcommand{\cR}{\mathcal{R}}

\newcommand{\R}{\mathbb{R}}
\newcommand{\N}{\mathbb{N}}
\newcommand{\true}{\text{true}}
\newcommand{\false}{\text{false}}

\newcommand{\setupto}[1]{[#1]} 
\newcommand{\map}[3]{#1\! : #2 \to #3}
\newcommand{\abs}[1]{|#1|}
\newcommand{\card}[1]{|#1|}
\newcommand{\with}{\! : \,}
\newcommand{\powerset}[1]{2^{#1}}
\newcommand{\bigo}[1]{\cO(#1)}

\newcommand{\maxisymb}{\texttt{max}}
\newcommand{\maxi}[1]{\maxisymb(#1)}

\newcommand{\aamdsymb}{\texttt{amd}}
\newcommand{\aamd}[1]{\aamdsymb(#1)} 
\newcommand{\mamdsymb}{\texttt{mmd}}
\newcommand{\mamd}[1]{\mamdsymb(#1)} 
\newcommand{\variancesymb}{\texttt{var}}
\newcommand{\variance}[1]{\variancesymb(#1)} 
\newcommand{\stdsymb}{\texttt{std}}
\newcommand{\std}[1]{\stdsymb(#1)} 
\newcommand{\rasmdsymb}{\texttt{rsm}}
\newcommand{\rasmd}[1]{\rasmdsymb(#1)} 
\newcommand{\mediansymb}{\texttt{med}}
\newcommand{\median}[1]{\mediansymb(#1)}
\newcommand{\averagesymb}{\texttt{mean}}
\newcommand{\average}[1]{\averagesymb(#1)}
\newcommand{\dispsymb}{d}
\newcommand{\disp}[1]{\dispsymb(#1)}
\newcommand{\tendsymb}{c}
\newcommand{\tend}[1]{\tendsymb(#1)}
\newcommand{\smdsymb}{\texttt{smd}}
\newcommand{\sumae}[1]{\smdsymb(#1)}
\newcommand{\mdccsymb}{\texttt{dcc}}
\newcommand{\mdcc}[1]{\mdccsymb(#1)}
\newcommand{\covsymb}{\texttt{cov}}
\newcommand{\cov}[1]{\covsymb(#1)}
\newcommand{\npmdssymb}{\texttt{mds}_+}
\newcommand{\npmds}[1]{\npmdssymb(#1)}
\newcommand{\impactsymb}{\mathtt{ipa}}
\newcommand{\impact}[1]{\impactsymb{(#1)}}

\newcommand{\oest}[1]{\hat{#1}}
\newcommand{\ext}[1]{\textbf{ext}(#1)}

\newcommand{\genlang}{\cL}
\newcommand{\cnjlang}{\genlang_{\text{cnj}}}
\newcommand{\genclosedlang}{\cC}
\newcommand{\cnjclosedlang}{\genclosedlang_{\text{cnj}}}

\newcommand{\succfunc}{\mathbf{r}}
\newcommand{\closuresymb}{\mathbf{c}}
\newcommand{\closure}[1]{\closuresymb(#1)}
\newcommand{\cnjclosuresymb}{\closuresymb_{\text{cnj}}}
\newcommand{\cnjclosure}[1]{\cnjclosuresymb(#1)}
\newcommand{\coreindex}[1]{\textbf{i}(#1)}
\newcommand{\lexprefix}[2]{#1\!\!\mid_{#2}}

\date{}
\begin{document}

\maketitle

\begin{abstract}
Existing algorithms for subgroup discovery with numerical targets do not optimize the error or target variable dispersion of the groups they find. This often leads to unreliable or inconsistent statements about the data, rendering practical applications, especially in scientific domains, futile. Therefore, we here extend the optimistic estimator framework for optimal subgroup discovery to a new class of objective functions: we show how tight estimators can be computed efficiently for all functions that are determined by subgroup size (non-decreasing dependence), the subgroup median value, and a dispersion measure around the median (non-increasing dependence). In the important special case when dispersion is measured using the mean absolute deviation from the median, this novel approach yields a linear time algorithm. Empirical evaluation on a wide range of datasets shows that, when used within branch-and-bound search, this approach is highly efficient and indeed discovers subgroups with much smaller errors.
\end{abstract}

\section{Introduction}
Subgroup discovery is a well-established KDD technique (\citet{klosgen1996explora, friedman1999bump, bay2001detecting}; see \citet{atzmueller2015subgroup} for a recent survey) with applications, e.g., in Medicine \citep{schmidt2010interpreting}, Social Science \citep{grosskreutz2010subgroup}, and Materials Science \citep{goldsmith2017uncovering}. In contrast to global modeling, which is concerned with the complete characterization of some variable defined for a given population, subgroup discovery aims to detect intuitive descriptions or selectors of subpopulations in which, \emph{locally}, the target variable takes on a useful distribution.
In scientific domains, like the ones mentioned above, 
such local patterns are typically considered useful if they are not too specific (in terms of subpopulation size) and indicate insightful facts about the underlying physical process that governs the target variable.
Such facts could for instance be: `patients of specific demographics experience a low response to some treatment' or `materials with specific atomic composition exhibit a high thermal conductivity'.
For numeric (metric) variables, subgroups need to satisfy two criteria to truthfully represent such statements: the local distribution of the target variable must have a shifted central tendency (effect), and group members must be described well by that shift (consistency). The second requirement is captured by the group's \emph{dispersion}, which determines the average error of associating group members with the central tendency value \citep[see also][]{song2016subgroup}.

\begin{figure}
\centering
\subfigure[optimizing coverage times median shift; $\sigma_{0}(\cdot) \equiv (a(\cdot) \geq 6) \wedge (c_2(\cdot) > 0) \wedge (c_6(\cdot)<\texttt{v. high})$;
subgroup median $0.22$, subgroup error $0.081$]{
\includegraphics[width=\pairedimagewidth]{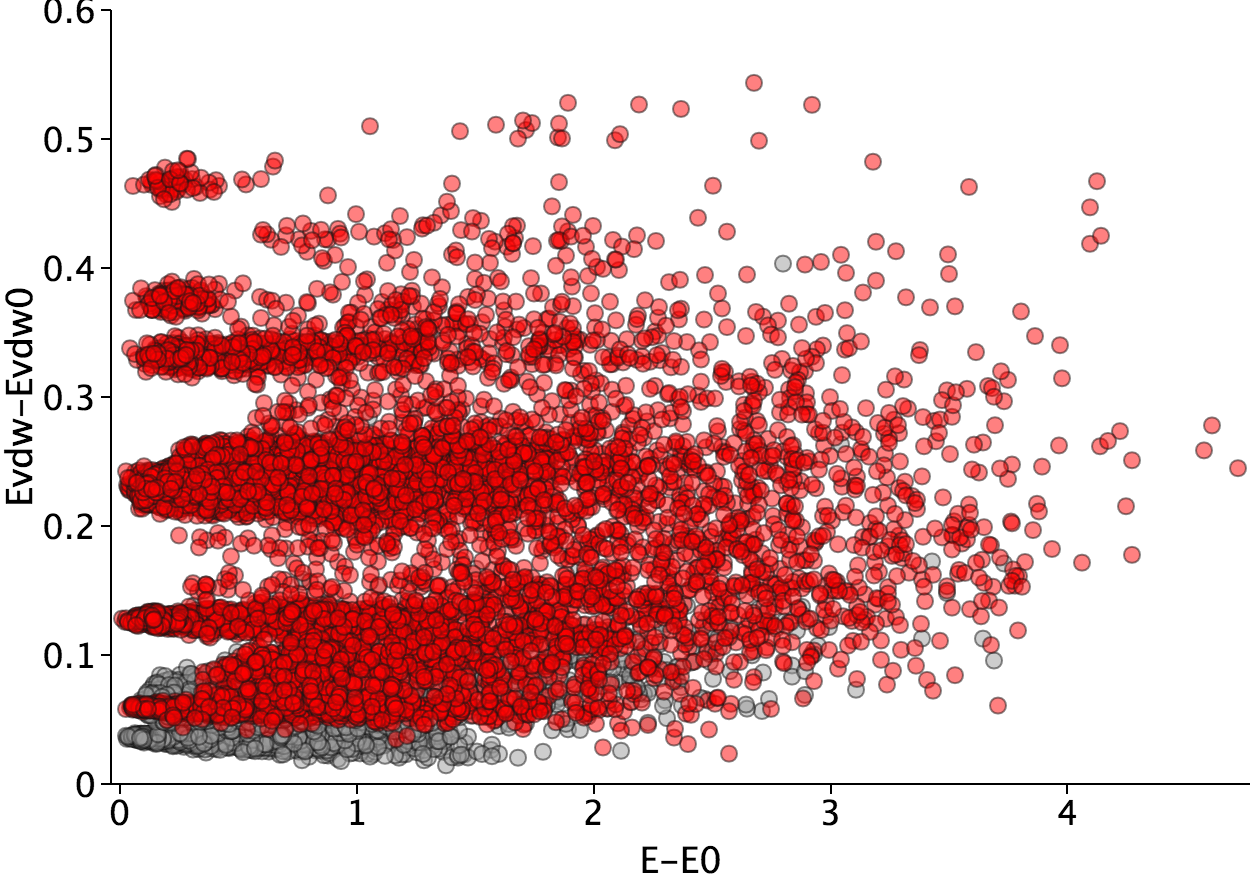}
}
\hfill
\subfigure[optimizing dispersion-corrected variant; $\sigma_{1}(\cdot) \equiv (a(\cdot) \in {[8,12]}) \wedge (c_2(\cdot)>\texttt{low}) \wedge (c_6(\cdot)<\texttt{v. high}) \wedge (r(\cdot)>\texttt{v. low})$;
subgroup median $0.23$, subgroup error $0.028$]{
\includegraphics[width=\pairedimagewidth]{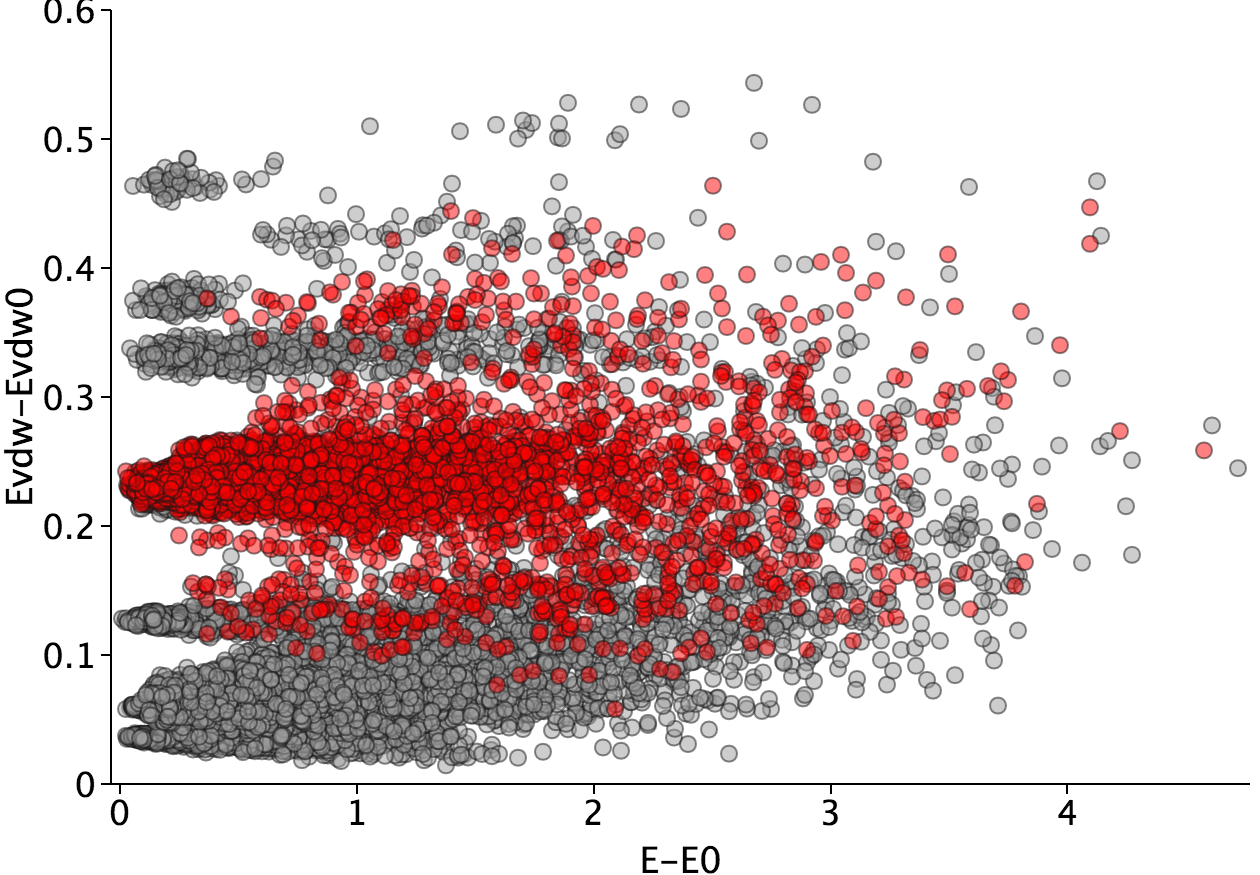}
}
\caption{To gain an understanding of the contribution of long-range
van der Waals interactions (y-axis; above) to the total energy (x-axis; above) of gas-phase gold
nanoclusters, subgroup discovery is used to analyze a
dataset of such clusters simulated ab initio by density
functional theory \citep{goldsmith2017uncovering}; available features describe nanocluster geometry and contain, e.g., \textit{number of atoms} $a$, \textit{fraction of atoms with $i$ bonds} $c_i$, and \textit{radius of gyration} $r$. Here, similar to other scientific
scenarios, a subgroup constitutes a useful piece of knowledge if it
conveys a statement about a remarkable amount of van der Waals
energy (captured by the group's central tendency) with high consistency
(captured by the group's dispersion/error);
optimal selector $\sigma_0$ with standard objective has high error and contains a large fraction of gold nanoclusters with a target value below the global median (0.13) \textbf{(a)}; this is not the case for selector $\sigma_1$ discovered through dispersion-corrected objective \textbf{(b)}, which therefore can be more consistently stated to describe gold nanoclusters with high van der Waals energy.}
\label{fig:illustration}
\end{figure}
Despite all three parameters---size, central tendency, and dispersion---being important, the only known approach for the efficient discovery of globally optimal subgroups, branch-and-bound search \citep{webb1995opus, wrobel1997algorithm}, is restricted to objective functions that only take into account size and central tendency.
That is, if we denote by $Q$ some subpopulation of our global population $P$ then the objective functions $f$ currently available to branch-and-bound can be written as 
\begin{equation}
f(Q)=g(\card{Q}, \tend{Q})
\label{eq:oldobj}
\end{equation}
where $\tendsymb$ is some measure of central tendency (usually mean or median) and $g$ is a function that is monotonically increasing in the subpopulation size $\card{Q}$.
A problem with \emph{all} such functions is that they inherently favor larger groups with scattered target values over smaller more focused groups with the same central tendency. 
That is, they favor the discovery of \emph{inconsistent} statements over consistent ones---surprisingly often identifying groups with a local error that is almost as high or even higher than the global error (see Fig.~\ref{fig:illustration} for an illustration of this problem that abounded from the authors' research in Materials Science).
Although \emph{dispersion-corrected} objective functions that counter-balance size by dispersion have been proposed (e.g., `$t$-score' by \citealp{klosgen2002data} or `mmad' by \citealp{pieters2010subgroup}),
it remained unclear how to employ such functions outside of heuristic optimization frameworks such as greedy beam search \citep{lavravc2004subgroup} or selector sampling \citep{boley2012linear, li2016sampling}.
Despite often finding interesting groups, such frameworks do not guarantee the detection of optimal results, which can not only be problematic for missing important discoveries but also because they therefore can never guarantee the \emph{absence} of high quality groups---which often is an insight equally important as the presence of a strong pattern.
For instance, in our example in Fig.~\ref{fig:illustration}, it would be remarkable to establish that long-range interactions are to a large degree independent of nanocluster geometry.

Therefore, in this paper (Sec.~\ref{sec:tightoests}), we extend branch-and-bound search to objective functions of the form
\begin{equation}
f(Q)=g(\card{Q}, \median{Q}, \disp{Q})
\label{eq:introdcobj}
\end{equation}
where $g$ is monotonically increasing in the subpopulation size, monotonically decreasing in any dispersion measure $\dispsymb$ around the median, and, besides that, depends only (but in arbitrary form) on the subpopulation median.
This involves developing an efficient algorithm for computing the \emph{tight optimistic estimator} given by the optimal value of the objective function among all possible subsets of target values:
\begin{equation}
\oest{f}(Q)=\max\{f(R)\with R \subseteq Q\} \enspace ,
\label{eq:tightoest}
\end{equation}
which has been shown to be a crucial ingredient for the practical applicability of branch-and-bound \citep{grosskreutz2008tight, lemmerich2016fast}.
So far, the most general approach to this problem (first codified in \citet{lemmerich2016fast}; generalized here in Sec.~\ref{sec:topseq}) is to maintain a sorted list of target values throughout the search process and then to compute Eq.~\eqref{eq:tightoest} as the maximum of all subsets $R_i \subseteq Q$ that contain all target values of $Q$ down to target value $i$---an algorithm that does not generalize to objective functions depending on dispersion.
This paper presents an alternative idea (Sec.~\ref{sec:medianseq})
where we do not fix the size of subset $R_i$ as in the previous approach but instead fix its median to target value $i$. It turns out that this suffices to efficiently compute the tight optimistic estimator for all objective functions of the form of Eq.~\eqref{eq:introdcobj}.
Moreover, we end up with a linear time algorithm (Sec.~\ref{sec:lineartime}) in the important special case where the dependence on size and dispersion is determined 
by the \emph{dispersion-corrected coverage} defined by
\[
\mdcc{Q}=\frac{\card{Q}}{\card{P}}\max \left\{1-\frac{\aamd{Q}}{\aamd{P}},0\right\}
\]
where $\aamdsymb$ denotes the mean absolute deviation from the median.
This is the same computational complexity as the objective function itself.
Consequently, this new approach can discover subgroups according to a more refined selection criterion without increasing the worst-case computational cost.
Additionally, as demonstrated by empirical results on a wide range of datasets (Sec.~\ref{sec:practice}), it is also highly efficient and
successfully reduces the error of result subgroups in practice.

\section{Subgroup Discovery}
\label{sec:prelim}
Before developing the novel approach to tight optimistic estimator computation, we recall in this section the necessary basics of optimal subgroup discovery with numeric target attributes. We focus on concepts that are essential from the optimization point of view (see, e.g., \citealt{duivesteijn2011exploiting} and references therein for statistical considerations). As notional convention, we are using the symbol $\setupto{m}$ for a positive integer $m$ to denote the set of integers $\{1,\dots,m\}$.
Also, for a real-valued expression $x$ we write $(x)_+$ to denote $\max\{x,0\}$.
A summary of the most important notations used in this paper can be found in Appendix~\ref{apx:notation}.

\subsection{Description languages, objective functions, and closed selectors}
\label{sec:langobjclosed}
Let $P$ denote our given \defemph{global population} of entities, for each of which we know the value of a real \defemph{target variable} $y\!:P \to \mathbb{R}$ and additional descriptive information that is captured in some abstract \defemph{description language} $\cL$ of subgroup selectors $\map{\sigma}{P}\{\true,\false\}$.
Each of these selectors describes a subpopulation $\ext{\sigma} \subseteq P$ defined by 
\[
\ext{\sigma}=\{p \in P \with \sigma(p)=\true\}
\]
that is referred to as the \defemph{extension} of $\sigma$.
Subgroup discovery is concerned with finding descriptions $\sigma \in \cL$ that have a useful (or interesting) distribution of target values in their extension $y_\sigma=\{y(p) : p \in \ext{\sigma}\}$.
This notion of usefulness is given by an \defemph{objective function} $\map{f}{\cL}{\R}$. That is, the formal goal is to find elements $\sigma \in \cL$ with maximal $f(\sigma)$.
Since we assume $f$ to be a function of the multiset of $y$-values, let us define $f(\sigma)=f(\ext{\sigma})=f(y_\sigma)$ to be used interchangeably for convenience. 
One example of a commonly used objective function is the \defemph{impact measure} $\impactsymb$ (see \citealt{webb2001discovering}; here a scaled but order-equivalent version is given) defined by
\begin{equation}
\impact{Q}=\cov{Q}\left ( \frac{\average{Q}-\average{P}}{\maxi{P}-\average{P}}\right )_+ 
\label{eq:impact}
\end{equation}
where $\cov{Q}=\card{Q}/\card{P}$ denotes the \defemph{coverage} or relative size of $Q$ (here---and wherever else convenient---we identify a subpopulation $Q \subseteq P$ with the multiset of its target values).

The standard description language in the subgroup discovery literature\footnote{In this article we remain with this basic setting for the sake of simplicity. It is, however, useful to note that several generalizations of this concept have been proposed \citep[e.g.,][]{parthasarathy1999incremental,huan2003efficient}, to which the contributions of this paper remain applicable.} is the language $\cnjlang$ consisting of \defemph{logical conjunctions} of a number of base propositions (or predicates).
That is, $\sigma \in \cnjlang$ are of the form
\[
\sigma(\cdot) \equiv \pi_{i_1}\!(\cdot) \wedge \dots \wedge \pi_{i_l}\!(\cdot)
\]
where the $\pi_{i_j}$ are taken from a pool of \defemph{base propositions} $\Pi=\{\pi_1,\dots,\pi_k\}$. These propositions usually correspond to equality or inequality constraints with respect to one variable $x$ out of a set of description variables $\{x_1,\dots,x_n\}$ that are observed for all population members (e.g., $\pi(p)\equiv x(p) \geq v$). However, for the scope of this paper it is sufficient to simply regard them as abstract Boolean functions $\map{\pi}{P}{\{\true,\false\}}$.
In this paper, we focus in particular on the refined language of \defemph{closed conjunctions} $\cnjclosedlang \subseteq \cnjlang$ \citep{pasquier1999efficient}, which is 
defined as $\cnjclosedlang=\{\sigma \in \cnjlang: \closure{\sigma}=\sigma\}$ by the fixpoints of the \defemph{closure operation} $\map{\closuresymb}{\cnjlang}{\cnjlang}$ given by 
\begin{equation}
\closure{\sigma} = \bigwedge \{\pi \in \Pi: \ext{\pi} \supseteq \ext{\sigma}\} \enspace .
\label{eq:closure}
\end{equation}
These are selectors to which no further proposition can be added without reducing their extension, and it can be shown that $\cnjclosedlang$ contains at most one selector for each possible extension.
While this can reduce the search space for finding optimal subgroups by several orders of magnitude, closed conjunctions are the longest (and most redundant) description for their extension and thus do not constitute intuitive descriptions by themselves. 
Hence, for reporting concrete selectors (as in Fig.~\ref{fig:illustration}), closed conjunctions have to be simplified to selectors of approximately minimum length that describe the same extension \citep{boley2009non}.

\subsection{Branch-and-bound and optimistic estimators}
\label{sec:bandb}
\SetKwFunction{bstbb}{Bst-BB}
\SetKwFunction{qempty}{empty}
\SetKwFunction{qaddall}{addAll}
\SetKwFunction{qtop}{top}
\SetKwFunction{suc}{suc}
\SetKwFunction{argmax}{argmax}
\begin{algorithm}[t]
\bstbb{$\cF,\sigma$}: \tcp*[r]{$\cF$ max. priority queue w.r.t.~$\oest{f}$, $\sigma$ current $f$-maximizer}
\Begin{
\eIf{$\cF=\emptyset$ {\bf or} $\oest{f}(\qtop{$\cF$})/f(\sigma) \leq a$}{\KwRet{$\sigma$}\;}
{
$\cR=\succfunc(\qtop{$\cF$})$ \tcp*[r]{refinement of $\oest{f}$-maximizer in queue}
$\sigma' = \argmax{$f(\varphi): \varphi \in \{\sigma\} \cup \cR$}$\;
$\cF'=(\cF \setminus \{\qtop{$\cF$}\}) \cup \{\varphi \in \cR \with \oest{f}(\varphi)/ f(\sigma')\geq a\}$\;
\KwRet{\bstbb{$\cF',\sigma'$}}\;
}
}
\BlankLine
$\sigma^*$ = \bstbb{$\{\bot\},\bot$}\tcp*[r]{call with root element to find global solution}
\caption{Best-first branch-and-bound that finds $a$-approximation to objective function $f$ based on refinement operator $\succfunc$ and optimistic estimator $\oest{f}$; depth-limit and multiple solutions (top-$k$) parameters omitted; \texttt{top} denotes the find max operation for priority queue.}
\label{alg:bb}
\end{algorithm}
The standard algorithmic approach for finding optimal subgroups with respect to a given objective function is branch-and-bound search---a versatile algorithmic puzzle solving framework with several forms and flavors \citep[see, e.g.,][Chap.~12.4]{mehlhorn2008algorithms}. At its core, all of its variants assume the availability and efficient computability of two ingredients: 
\begin{enumerate}
\item A \defemph{refinement operator} $\map{\succfunc}{\cL}{2^\cL}$ that is monotone, i.e., for $\sigma, \varphi \in \cL$ with $\varphi \in \succfunc(\sigma)$ it holds that $\ext{\varphi} \subseteq \ext{\sigma}$,
and that non-redundantly generates $\cL$.
That is, there is a root selector $\bot \in \cL$ such that for every $\sigma \in \cL$ there is a unique sequence of selectors $\bot=\sigma_0, \sigma_1, \dots, \sigma_l=\sigma$ with $\sigma_{i} \in \succfunc(\sigma_{i-1})$. In other words, the refinement operator implicitly represents a directed tree (arborescence) on the description language $\cL$ rooted in $\bot$. 
\item An \defemph{optimistic estimator} (or bounding function) $\map{\oest{f}}{\cL}{\R}$ that bounds from above the attainable subgroup value 
of a selector among all more specific selectors, i.e., it holds that $\oest{f}(\sigma) \geq f(\varphi)$ for all $\varphi \in \cL$ with $\ext{\varphi} \subseteq \ext{\sigma}$.
\end{enumerate}
Based on these ingredients, a branch-and-bound algorithm simply enumerates all elements of $\cL$ starting from $\bot$ using $\succfunc$ (branch), but---based on $\oest{f}$---avoids expanding descriptions that cannot yield an improvement over the best subgroups found so far (bound). Depending on the order in which language elements are expanded, one distinguishes between depth-first, breadth-first, breadth-first iterating deepening, and best-first search.
In the last variant, the optimistic estimator is not only used for pruning the search space, but also to select the next element to be expanded, which is particularly appealing for informed, i.e., tight, optimistic estimators.
An important feature of branch-and-bound is that it effortlessly allows to speed-up the search in a sound way by relaxing the result requirement from being $f$-optimal to just being an \defemph{$a$-approximation}. That is, the found solution $\sigma$ satisfies for all $\sigma' \in \cL$ that $f(\sigma)/f(\sigma') \geq a$ for some \defemph{approximation factor} $a \in (0,1]$.
The pseudo-code given in Alg.~\ref{alg:bb} summarizes all of the above ideas.
Note that, for the sake of clarity, we omitted here some other common parameters such as a depth-limit and multiple solutions (top-$k$), which are straightforward to incorporate \citep[see][]{lemmerich2016fast}.

An efficiently computable refinement operator has to be constructed specifically for the desired description language. For example for the language of conjunctions $\cnjlang$, one can define $\map{\succfunc_{\text{cnj}}}{\cnjlang}{{\cnjlang}}$ by
\[
\succfunc_{\text{cnj}}(\sigma)=\{\sigma \wedge \pi_i \with \max \{j\with \pi_j \in \sigma\} < i \leq k\} 
\]
where we identify a conjunction with the set of base propositions it contains. 
For the closed conjunctions $\cnjclosuresymb$, 
let us define the lexicographical prefix of a conjunction $\sigma \in \cnjlang$ and a base proposition index $i \in \setupto{k}$ as $\lexprefix{\sigma}{i}=\sigma \cap \{\pi_1, \dots, \pi_i\}$. Moreover, let us denote with $\coreindex{\sigma}$ the minimal index such that the $i$-prefix of $\sigma$ is extension-preserving, i.e.,
$
\coreindex{\sigma}=\min\{i \with \ext{\lexprefix{\sigma}{i}}=\ext{\sigma}\}$.
With this we can construct a refinement operator \citep{uno2004efficient} $\map{\succfunc_{\text{ccj}}}{\cnjclosedlang}{\powerset{\cnjclosedlang}}$ as
\[
\succfunc_{\text{ccj}}(\sigma)=\{\varphi \with \varphi=\cnjclosure{\sigma \wedge \pi_j}, \coreindex{\sigma} < j \leq k, \pi_j \not\in \sigma, \lexprefix{\varphi}{j}=\lexprefix{\sigma}{j}\} \enspace .
\]
That is, a selector $\varphi$ is among the refinements of $\sigma$ if $\varphi$ can be generated by an application of the closure operator given in Eq.~\eqref{eq:closure} that is prefix-preserving.

How to obtain an optimistic estimator for an objective function of interest depends on the definition of that objective.
For instance, the coverage function $\covsymb$ is a valid optimistic estimator for the impact function $\impactsymb$ as defined in Eq.~\eqref{eq:impact}, because the second factor of the impact function is upper bounded by $1$.
In fact there are many different optimistic estimators for a given objective function.
Clearly, the smaller the value of the bounding function for a candidate subpopulation, the higher is the potential for pruning the corresponding branch from the enumeration tree.
Ideally, one would like to use $\oest{f}(\sigma)=\max \{f(\varphi) \with \ext{\varphi} \subseteq \ext{\sigma} \}$, which is the most strict function that still is a valid optimistic estimator.
In general, however, computing this function is as hard as the original subgroup optimization problem we started with.
Therefore, as a next best option, one can disregard selectability and consider the (selection-unaware) \defemph{tight optimistic estimator} given by
\[
\oest{f}(\sigma)=\max\{f(R)\with R \subseteq \ext{\sigma}\} \enspace .
\]
This leaves us with a new combinatorial optimization problem: given a subpopulation $Q \subseteq P$, find a sub-selection of $Q$ that maximizes $f$. In the following section we will discuss strategies for solving this optimization problem efficiently for different classes of objective functions---including dispersion-corrected objectives. 

\newcommand{\domby}{\leq_e}
\newcommand{\diffsetmed}[1]{#1_{\Delta}^{\mediansymb}}
\section{Efficiently Computable Tight Optimistic Estimators}
\label{sec:tightoests}
We are going to develop an efficient algorithm for the tight optimistic estimator in three steps: First, we review and reformulate a general algorithm for the classic case of non-dispersion-aware objective functions. Then we transfer the main idea of this algorithm to the case of dispersion-corrected objectives based on the median, and finally we consider a subclass of these functions where the approach can be computed in linear time.
Throughout this section we will identify a given subpopulation $Q \subseteq P$ with the multiset of its target values $\{y_1,\dots,y_m\}$ and assume that the target values are indexed in ascending order, i.e., $y_i \leq y_j$ for $i \leq j$.
Also, it is helpful to define the following partial order defined on finite multisets.
Let $Y=\{y_1,\dots,y_m\}$ and $Z=\{z_1,\dots,z_{m'}\}$ be two multisets such that their elements are indexed in ascending order. We say that $Y$ is \defemph{element-wise less or equal} to $Z$ and write $Y \domby Z$ if $y_i \leq z_i$ for all $i \in \setupto{\min\{m,m'\}}$.

\subsection{The standard case: monotone functions of a central tendency measure}
\label{sec:topseq}
The most general previous approach for computing the tight optimistic estimator for subgroup discovery with a metric target variable is described by 
\citet{lemmerich2016fast} where it is referred to as \emph{estimation by ordering}. Here, we review this approach and give a uniform and generalized version of that paper's results.
For this, we define the general notion of a measure of central tendency as follows.
\begin{definition}
We call a mapping $\map{\tendsymb}{\N^\R}{\R}$ a (monotone) \defemph{measure of central tendency} if for all multisets $Y,Z \in \N^\R$ with $Y \domby Z$ it holds that $\tend{Y} \leq \tend{Z}$.
\end{definition}
One can check that this definition applies to the standard measures of central tendency, i.e., the arithmetic and geometric mean as well as the \defemph{median}\footnote{In this paper, we are using the simple definition of the median as the $0.5$-quantile (as opposed to defining it as $(y_{m/2}+y_{1+m/2})/2$ for even $m$), which simplifies many of the definitions below and additionally is well-defined in settings where averaging of target values is undesired.} $\median{Q}=y_{\lceil m/2 \rceil}$, and also to weighted variants of them (note, however, that it does not apply to the mode).
With this we can define the class of objective functions for which the tight optimistic estimator can be computed efficiently by the standard approach as follows.
We call $\map{f}{\powerset{P}}{\R}$ a \defemph{monotone level 1 objective function} if it can be written as
\[
f(Q)=g(\card{Q}, \tend{Q})
\]
where $\tendsymb$ is some measure of central tendency and $g$ is a function that is non-decreasing in both of its arguments. One can check that the impact measure $\impactsymb$ falls under this category of functions as do many of its variants.

The central observation for computing the tight optimistic estimator for monotone level 1 functions is that the optimum value must be attained on a sub-multiset that contains a consecutive segment of elements of $Q$ from the top element w.r.t.\ $y$ down to some cut-off element. Formally, let us define the \defemph{top sequence} of sub-multisets of $Q$ as
$
T_i=\{y_{m-i+1},\dots, y_m\}
$
for $i \in \setupto{m}$ and note the following observation:
\begin{prop}
Let $f$ be a monotone level 1 objective function.
Then the tight optimistic estimator of $f$ can be computed as the maximum value on the top sequence, i.e.,
$
\oest{f}(Q)=\max \{f(T_i) \with i \in \setupto{m}\}
$.
\label{prop:topseq}
\end{prop}
\begin{proof}
Let $R \subseteq Q$ be of size $k$ with $R=\{y_{i_1},\dots,y_{i_k}\}$. Since $y_{i_j}\leq y_{m-j+1}$, we have for the top sequence element $T_k$ that $R \domby T_k$ and, hence, $\tend{R} \leq \tend{T_k}$ implying
\[
f(R)=g(k,\tend{R}) \leq g(k,\tend{T_l})=f(T_k) \enspace .
\]
It follows that for each sub-multiset of $Q$ there is a top sequence element of at least equal objective value.
\smartqed
\end{proof}
From this insight it is easy to derive an $\bigo{m}$ algorithm for computing the tight optimistic estimator under the additional assumption
that we can compute $g$ and the ``incremental central tendency problem'' $(i,Q,(\tend{T_1},\dots,\tend{T_{i-1}}) \mapsto\tend{T_i}$ in constant time.
Note that computing the incremental problem in constant time implies to only access a constant number of target values and of the previously computed central tendency values. This can for instance be done for $\tendsymb=\averagesymb$ via the incremental formula $\average{T_i}=((i-1)\,\average{T_{i-1}}+y_{m-i+1})/i$ or for $\tendsymb=\mediansymb$ through direct index access of either of the two values $y_{m-\lfloor(i-1)/2 \rfloor}$ or $y_{m-\lceil (i-1)/2 \rceil}$.
Since, according to Prop.~\ref{prop:topseq}, we have to evaluate $f$ only for the $m$ candidates $T_i$ to find $\oest{f}(Q)$ we can do so in time $\bigo{m}$ by solving the problem incrementally for $i=1,\dots,m$.
The same overall approach can be readily generalized for objective functions that are monotonically decreasing in the central tendency or those that can be written as the maximum of one monotonically increasing and one monotonically decreasing level 1 function. However, it breaks down for objective functions that depend on more than just size and central tendency---which inherently is the case when we want to incorporate dispersion-control.

\subsection{Dispersion-corrected objective functions based on the median}
\label{sec:medianseq}
We will now extend the previous recipe for computing the tight optimistic estimator to objective functions that 
depend not only on subpopulation size and central tendency but also on the target value dispersion in the subgroup.
Specifically, we focus on the median as measure of central tendency and consider functions that are both monotonically increasing in the described subpopulation size and monotonically decreasing in some dispersion measure around the median.
To precisely describe this class of functions, we first have to formalize the notion of dispersion measure around the median.
For our purpose the following definition suffices.
Let us denote by $\diffsetmed{Y}$ the \defemph{multiset of absolute differences} to the median of a multiset $Y \in \N^\R$, i.e., $\diffsetmed{Y}=\{\abs{y_1-\median{Y}},\dots,\abs{y_m-\median{Y}}\}$.
\begin{definition}
\label{def:dsp}
We call a mapping $\map{\dispsymb}{\N^\R}{\R}$ a \defemph{dispersion measure around the median} if $\disp{Y}$ is monotone with respect to the multiset of absolute differences to its median $\diffsetmed{Y}$, i.e., if $\diffsetmed{Y} \domby \diffsetmed{Z}$ then $\disp{Y} \leq \disp{Z}$.
\end{definition}
One can check that this definition contains the measures median absolute deviation around the median $\mamd{Y}=\median{\diffsetmed{Y}}$, the root mean of squared deviations around the median $\rasmd{Y}=\average{\{x^2 \with x \in \diffsetmed{Y}\}}^{1/2}$, as well as the \defemph{mean absolute deviation around the median} $\aamd{Y}=\average{\diffsetmed{Y}}$.\!\footnote{We work here with the given definition of dispersion measure because of its simplicity. Note, however, that all subsequent arguments can be extended in a straightforward way to a wider class of dispersion measures by considering the multisets of positive and negative deviations separately. This wider class also contains the interquartile range and certain asymmetric measures, which are not covered by Def.~\ref{def:dsp}.}
Based on Def.~\ref{def:dsp} we can specify the class of objective functions that we aim to tackle as follows: we call a function $\map{f}{\powerset{P}}{\R}$ a \defemph{dispersion-corrected or level 2 objective function} (based on the median) if it can be written as
\begin{equation}
\label{eq:dspcorr}
f(Q)=g(\card{Q},\median{Q},\disp{Q})
\end{equation}
where $\dispsymb$ is some dispersion measure around the median and $\map{g}{\R^3}{\R}$ is a real function that is non-decreasing in its first argument and non-increasing in its third argument (without any monotinicity requirement for the second argument).

\begin{figure}[t]
\centering
\includegraphics[width=\columnwidth]{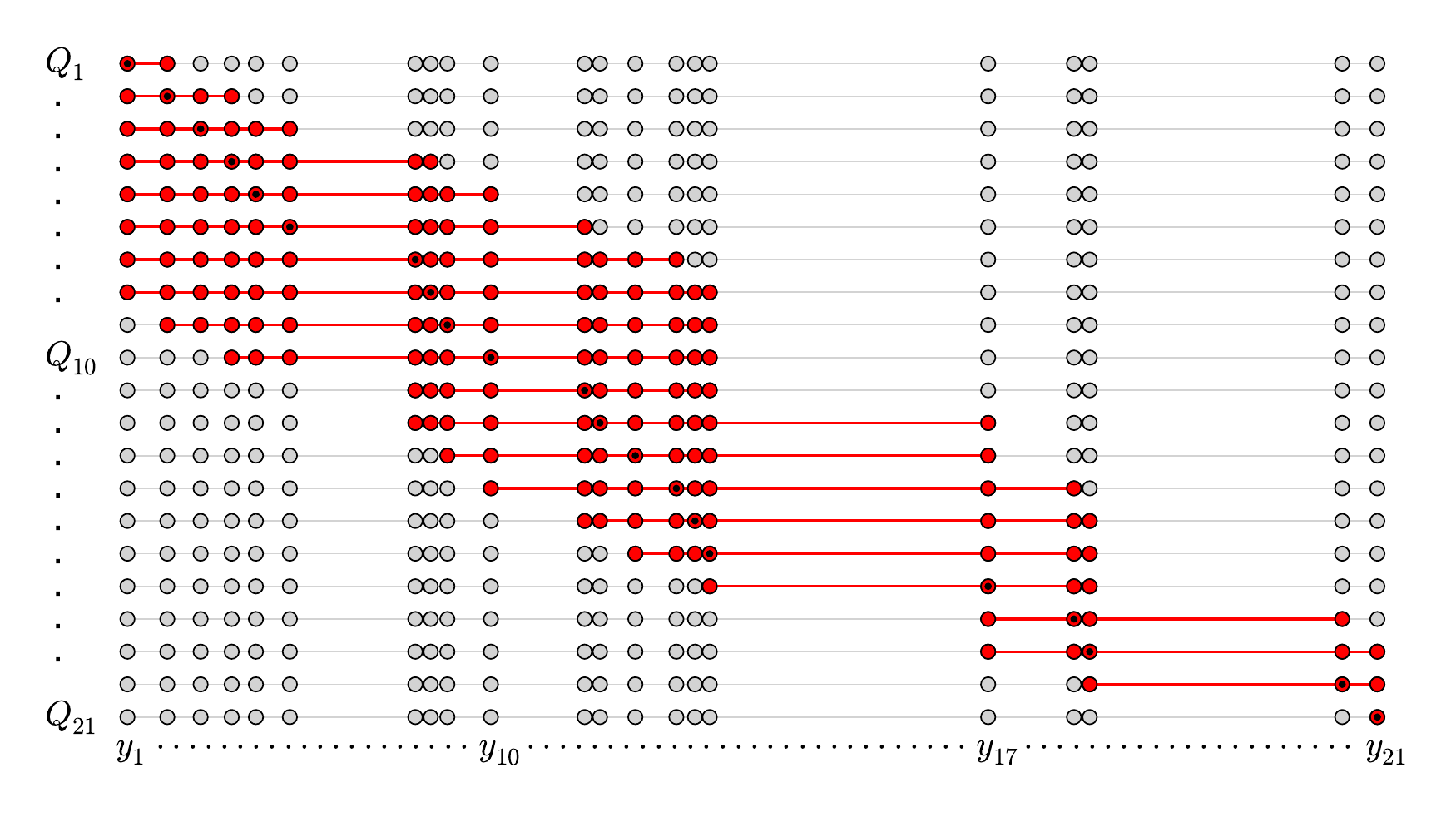}
\vspace{-0.5cm}
\caption{Median sequence sets $Q_1,\dots,Q_{21}$ (marked red) for $21$ random target values $y_{1}, \dots, y_{21}$ w.r.t. objective function $f(Q)=\card{Q}/\card{P}-\aamd{Q}/\aamd{P}$; the sets are identical for any arbitrary dependence on the median $\median{Q}$ that could potentially be added to $f$, and for any such function the optimal value is attained among those $21$ sets (Prop.~\ref{prop:medianseq}).}
\label{fig:med_seq_sets}
\end{figure}
Our recipe for optimizing these functions is then to consider only subpopulations $R \subseteq Q$ that can be formed by selecting all individuals with a target value in some interval. Formally, for a fixed index $z \in \{1,\dots,m\}$ define $m_z \leq m$ as the maximal cardinality of a sub-multiset of the target values that has median index $z$, i.e.,
\begin{equation}
m_z=\min\{2z,2(m-z)+1\} \enspace .
\label{eq:mz}
\end{equation}
Now, for $k \in \setupto{m_z}$, let us define $Q^k_{z}$ as the set with $k$ \emph{consecutive} elements around index $z$.
That is 
\begin{equation}
Q^k_z=\left\{y_{z-\left \lfloor \frac{k-1}{2} \right \rfloor}, \dots, y_z, \dots, y_{z+ \left \lceil \frac{k-1}{2} \right \rceil}\right\}  \enspace .
\label{eq:consecset}
\end{equation}
With this we can define the elements of the \defemph{median sequence} $Q_z$ as those subsets of the form of Eq.~\eqref{eq:consecset} that maximize $f$ for some fixed index $z \in \setupto{m}$.
That is, $Q_z=Q^{k^*_z}_z$ where $k^*_z \in \setupto{m_z}$ is minimal with
\[
f(Q^{k^*_z}_z)=g(k^*_z, y_z,\disp{Q^{k^*_z}_z})=\max \{f(Q^k_z) \with k \in \setupto{m_z}\} \enspace .
\]
Thus, the number $k^*_z$ is the smallest cardinality that maximizes the trade-off of size and dispersion encoded by $g$ (given the fixed median $y_z=\median{Q^k_z}$ for all $k$).
Fig.~\ref{fig:med_seq_sets} shows an exemplary median sequence based on $21$ random target values.
In the following proposition we note that, as desired, searching the median sequence is sufficient for finding optimal subsets of $Q$.
\begin{prop}
Let $f$ be a dispersion-corrected objective function based on the median.
Then the tight optimistic estimator of $f$ can be computed as the maximum value on the median sequence, i.e.,
$
\oest{f}(Q)=\max \{f(Q_z) \with z \in \setupto{m}\}
$.
\label{prop:medianseq}
\end{prop}
\begin{proof}
For a sub-multiset $R \subseteq Q$ let us define the gap count $\gamma(R)$ as 
\[
\gamma(R)=\card{\{y \in Q \setminus R \with \min R < y < \max R\}} \enspace .
\]
Let $O \subseteq Q$ be an $f$-maximizer with minimal gap count, i.e., $f(R) < f(O)$ for all $R$ with $\gamma(R)<\gamma(O)$. 
Assume that $\gamma(O)>0$. That means there is a $y \in Q \setminus O$ such that $\min O < y < \max O$. Define
\[
S=
\begin{cases}
(O \setminus \{\min O\}) \cup \{y\}, &\text{if $y \leq \median{O}$}\\
(O \setminus \{\max O\}) \cup \{y\}, &\text{otherwise}
\end{cases}
\enspace .
\]
Per definition we have $\card{S}=\card{O}$ and $\median{S}=\median{O}$. Additionally, we can check that $\diffsetmed{S} \domby \diffsetmed{O}$, and, hence, $\disp{S} \leq \disp{Q}$. This implies that 
\[
f(S)=g(\card{S},\median{S},\disp{S})\geq g(\card{O},\median{O},\disp{O})=f(O) \enspace .
\]
However, per definition of $S$ it also holds that $\gamma(S) < \gamma(O)$, which contradicts that $O$ is an $f$-optimizer with minimal gap count. Hence, any $f$-maximizer $O$ must have a gap count of zero. In other words, $O$ is of the form $O=Q^k_z$ as in Eq.~\eqref{eq:consecset} for some median $z \in \setupto{m}$ and some cardinality $k \in \setupto{m_z}$ and per definition we have $f(Q_z) \geq f(O)$ as required. \smartqed
\end{proof}
Consequently, we can compute the tight optimistic estimator for any dispersion-corrected objective function based on the median in time $\bigo{m^2}$ for subpopulations of size $m$---again, given a suitable incremental formula for $\dispsymb$.
While this is not generally a practical algorithm in itself, it is a useful departure point for designing one.
In the next section we show how it can be brought down to linear time when we introduce some additional constraints on the objective function.

\begin{figure}[t]
\centering
\includegraphics[width=\columnwidth]{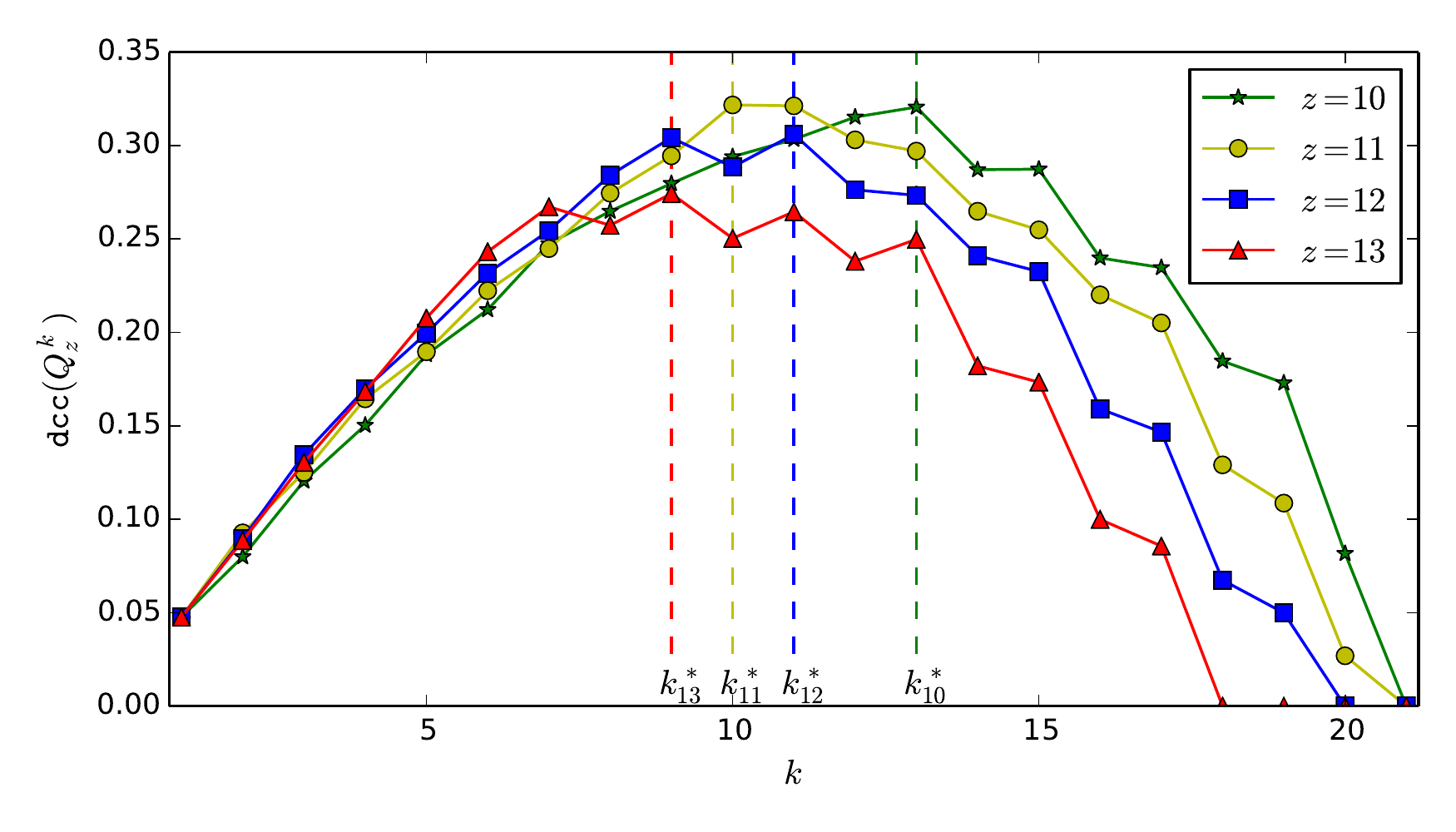}
\vspace{-0.5cm}
\caption{Dispersion-corrected coverage of the sets $Q^k_z$ as defined in Eq.~\eqref{eq:consecset} for median indices $z \in \{10,11,12,13\}$ and the $21$ random target values from Fig.~\ref{fig:med_seq_sets}; the sets $Q_z$ can be found in incremental constant time since optimal size $k^*_z$ is within a constant range of $k^*_{z+1}$ (Thm.~\ref{thm:to7}).}
\label{fig:med_seq_opt}
\end{figure}
\subsection{Reaching linear time---objectives based on dispersion-corrected coverage}
\label{sec:lineartime}
Equipped with the general concept of the median sequence, we can now address the special case of 
dispersion-corrected objective functions where the trade-off between the subpopulation size and target value dispersion is captured by a linear function of size and the sum of absolute differences from the median. Concretely, let us define the \defemph{dispersion-corrected coverage} (w.r.t. absolute median deviation) by
\[
\mdcc{Q}=\frac{\card{Q}}{\card{P}}\left(1-\frac{\aamd{Q}}{\aamd{P}}\right)_+ = \left ( \frac{\card{Q}}{\card{P}}-\frac{\sumae{Q}}{\sumae{P}} \right)_+
\]
where $\sumae{Q}=\sum_{y \in Q}\abs{y-\median{Q}}$ denotes the \defemph{sum of absolute deviations from the median}. We then consider objective functions based on the dispersion-corrected coverage of the form
\begin{equation}
f(Q)=g(\mdcc{Q},\median{Q})
\label{eq:dccobj}
\end{equation}
where $g$ is non-decreasing in its first argument. Let us note, however, that we could replace the $\mdccsymb$ function by any linear function that depends positively on $\card{Q}$ and negatively on $\smdsymb$. It is easy to verify that function of this form also obey the more general definition of level-2 objective functions given in Sec.~\ref{sec:medianseq}, and, hence can be optimized via the median sequence.

The key to computing the tight optimistic estimator $\oest{f}$ in linear time for functions based on dispersion-corrected coverage is then that the members of the median sequence $Q_z$ can be computed incrementally in constant time. Indeed, we can prove the following theorem, which states that the optimal size for a multiset around median index $z$ is within $3$ of the optimal size for a multiset around median index $z+1$---a fact that can also be observed in the example given in Fig.~\ref{fig:med_seq_sets}.
\begin{theorem}
\label{thm:to7}
Let $f$ be of the form of Eq.~\eqref{eq:dccobj}. For $z \in \setupto{m-1}$ it holds for the size $k_z^*$ of the $f$-optimal multiset with median $z$ that 
\begin{equation}
k_{z}^* \in \{\max(0,k^*_{z+1}-3),\dots,\min(m_{z},k^*_{z+1}+3)\} \enspace .
\label{eq:to7}
\end{equation}
\end{theorem}
One idea to prove this theorem is to show that a) the gain in $f$ for increasing the multiset around a median index $z$ is alternating between two discrete concave functions and b) that the gains for growing multisets between two consecutive median indices are bounding each other.
For an intuitive understanding of this argument, Fig.~\ref{fig:med_seq_opt} shows for four different median indices $z \in \{10,11,12,13\}$ the dispersion-corrected coverage for the sets $Q^k_z$ as a function in $k$. On closer inspection, we can observe that when considering only every second segment of each function graph, the corresponding $\mdccsymb$-values have a concave shape.
A detailed proof, which is rather long and partially technical, can be found in Appendix~\ref{apx:proofto7}.

\begin{algorithm}[t]
{\bf let} $Q$ be given by $\{y_1,\dots,y_m\}$ in ascending order\\
{\bf compute} $e_l(i)$ and $e_r(i)$ for $i \in \setupto{m}$ through Eqs.~\eqref{eq:lerec} and \eqref{eq:rerec}\\
$f(Q_m)=g(1/\card{P},y_m)$ and $k^*_m=1$\\
\For{$z = m-1$ \KwTo $1$}{
{\bf let} $k^-=\max(0,k^*_{z+1}-3)$ and $k^+=\min(m_{z},k^*_{z+1}+3)$ with $m_z$ as in Eq.~\eqref{eq:mz}\\
\For{$k = k^-$ \KwTo $k^+$}{
{\bf let} $a=z-\left\lfloor k/2 \right \rfloor$ and $b=z+\left\lceil k/2 \right\rceil$\\
$\sumae{Q^k_z}=e_l(z)-e_l(a)-(a-1)(y_z-y_a)+e_r(z)-e_r(b)-(m-b)(y_b-y_z)$\\
$f(Q^k_z)=g(k/\card{P}-\sumae{Q^k_z}/\sumae{P}, y_z)$
}
$f(Q_z)=f(Q^{k_z^*}_z)$ with $k^*_z$ s.t. $f(Q^{k_z^*}_z)=\max \{f(Q^k_z) \with k^- \leq k \leq k^+\}$\\
}
$\oest{f}(Q)=\max \{f(Q_z) \with z \in \setupto{m}\}$
\caption{Linear time algorithm for computing tight optimistic estimator $\oest{f}(Q)$ of objective $f(Q)=g(\mdcc{Q},\median{Q})$ as in Eq.~\eqref{eq:dccobj}.}
\label{alg:oest}
\end{algorithm}
It follows that, after computing the objective value of $Q_m$ trivially as $f(Q_m)=g(1/\card{P},y_m)$, we can obtain $f(Q_{z-1})$ for $z=m,\dots,2$ by checking the at most seven candidate set sizes given by Eq.~\eqref{eq:to7} as
\[
f(Q^{z-1})=\max \left\{f(Q^{k^-_z}_{z-1}),\dots, f(Q^{k_z^+}_{z-1})\right\}
\]
with $k_z^-=\max(k_z^*-3,1)$ and $k_z^+=\min(k_z^*+3,m_z)$.
It remains to see that we can compute individual evaluations of $f$ in constant time (after some initial $\bigo{m}$ pre-processing step). As a general data structure for quickly computing sums of absolute deviations from a center point, we can define for $i \in \setupto{m}$ the \defemph{left error} $e_l(i)$ and the \defemph{right error} $e_r(i)$ as
\[
e_l(i)=\sum_{j=1}^{i-1} y_i-y_j, \hspace{1cm} e_r(i)=\sum_{j=i+1}^{m} y_j-y_i \enspace .
\]
Note that we can compute these error terms for all $i \in \setupto{m}$ in time $\bigo{m}$ via the recursions
\begin{align}
e_l(i)&=e_l(i-1)+(i-1)(y_i-y_{i-1})\label{eq:lerec} \\
e_r(i)&=e_r(i+1)+(m-i)(y_{i+1}-y_i)\label{eq:rerec}
\end{align}
and $e_l(1)=e_r(m)=0$. Subsequently, we can compute sums of deviations from center points of arbitrary subpopulations in constant time, as the following statement shows (see Appendix~\ref{apx:addproofs} for a proof).
\begin{prop}
\label{prop:fastsumcomp}
Let $Q=\{y_1,\dots,y_a, \dots, y_z, \dots, y_b,\dots, y_m\}$ be a multiset with $1 \leq a<z<b \leq m$ and $y_i\leq y_{j}$ for $i\leq j$. Then the sum of absolute deviations to $y_i$ of all elements of the submultiset $\{y_a, \dots, y_z, \dots, y_b\}$ can be expressed as
\begin{align*}
\sum_{i=a}^b \abs{y_z-y_i}&=e_l(z)-e_l(a)-(a-1)(y_z-y_a)\\
&\hspace{0.5cm}+e_r(z)-e_r(b)-(m-b)(y_b-y_z) \enspace .
\end{align*}
\end{prop}
With this we can compute $k \mapsto f(Q_z^k)$ in constant time (assuming $g$ can be computed in constant time). Together with Prop.~\ref{prop:medianseq} and Thm.~\ref{thm:to7} this results in a linear time algorithm for computing $Q \mapsto \oest{f}(Q)$ (see Alg.~\ref{alg:oest} for a pseudo-code that summarizes all ideas).

\begin{figure}[t]
\centering
\includegraphics[width=\columnwidth]{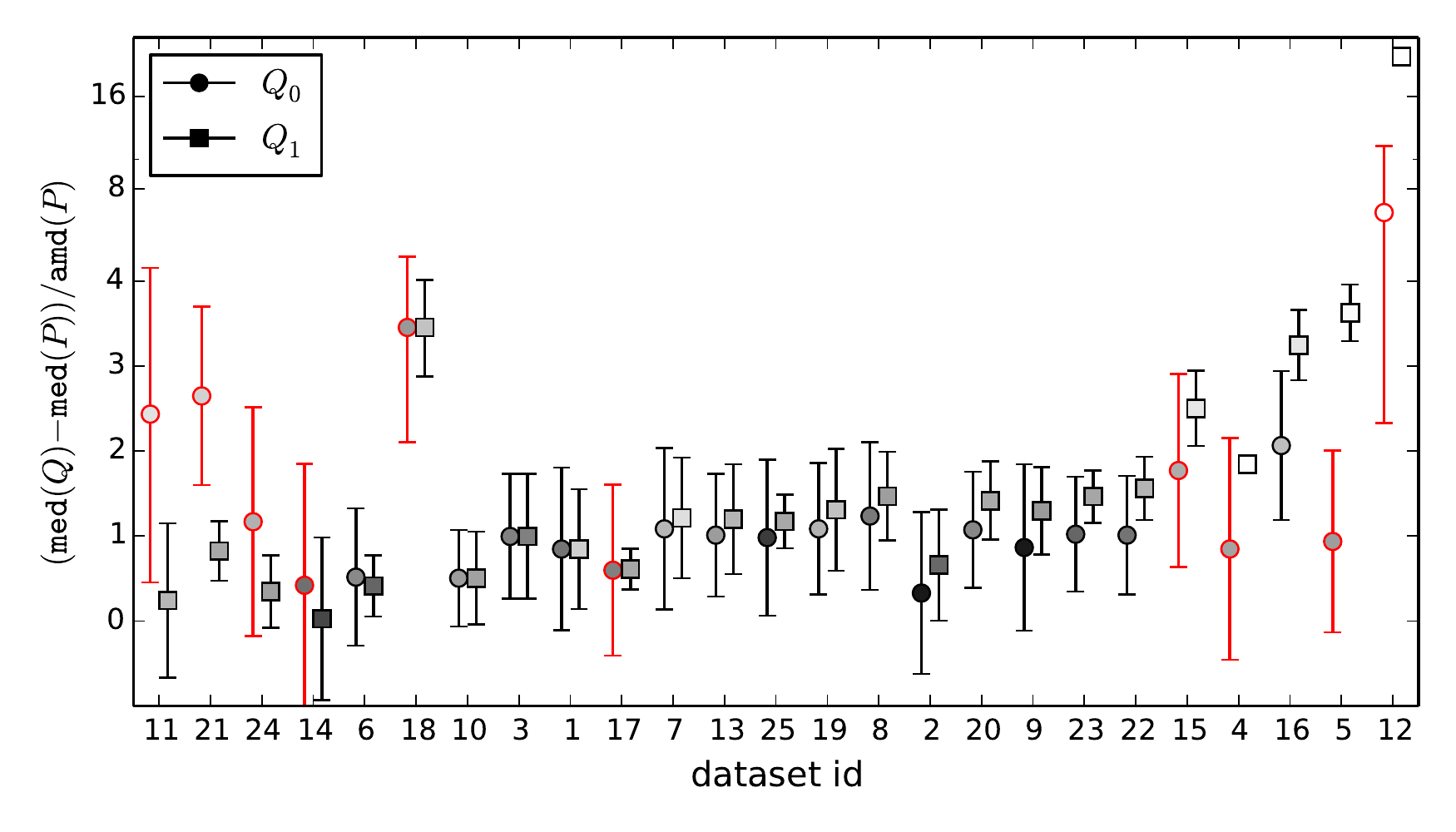}
\vspace{-0.5cm}
\caption{Normalized median of optimal subgroup w.r.t.\ uncorrected positive median shift $(Q_0)$ and w.r.t.\ dispersion-corrected positive median shift ($Q_1$) for 25 test datasets, sorted according to median difference. Error bars show mean absolute median deviation of subgroups; groups marked red have larger deviation than global deviation; fill color indicates group coverage from 0 (white) to 1 (black).}
\label{fig:selection_bias}
\end{figure}
\section{Dispersion-corrected Subgroup Discovery in Practice}
\label{sec:practice}
\newcommand{\bestwrm}{Q_0}
\newcommand{\bestdcm}{Q_1}
\newcommand{\nodestight}{\cE_1}
\newcommand{\nodesloose}{\cE_0}
\newcommand{\numnodestight}{\card{\nodestight}}
\newcommand{\numnodesloose}{\card{\nodesloose}}
\newcommand{\pot}[2]{#1\!\times\!10^{#2}}
\begin{sidewaystable*}[ph!]
\bigskip
 \centering\small\setlength\tabcolsep{2pt}
\resizebox{\columnwidth}{!}{%
\begin{tabular}{rllrrrrrrrrrrrrrrr}\toprule
\multicolumn{7}{c}{Dataset} & \multicolumn{6}{c}{Selection Bias} & \multicolumn{5}{c}{Efficiency}\\
\cmidrule(lr){1-7} \cmidrule(lr){8-13} \cmidrule(lr){14-18}
\,  & Name        & Target & $\card{P}$ & $\card{\Pi}$ & $\median{P}$ & $\aamd{P}$ & $\cov{\bestwrm}$ & $\cov{\bestdcm}$ & $\median{\bestwrm}$ & $\median{\bestdcm}$ & $\aamd{\bestwrm}$      & $\aamd{\bestdcm}$   & $a_\text{eff}$ & $\numnodesloose$ & $\numnodestight$ & $t_0$               & $t_1$               \\ \midrule
1   & abalone     & rings & 4,177      & 69           & 9            & 2.359      & $\mathbf{0.544}$ & 0.191            & 11                  & 11                  & 2.257                  & $\mathbf{1.662}$    & 1              & $848,258$        & $690,177$        & $\mathbf{304}$      & $339$               \\
2   & ailerons    & goal & 13,750     & 357          & $-0.0008$    & $0.000303$ & $\mathbf{0.906}$ & 0.59             & $-0.0007$           & $\mathbf{-0.0006}$  & $0.000288$             & $\mathbf{0.000198}$ & 0.3            & $1,069,456$      & $54,103$         & $6,542$             & $\mathbf{460}$      \\
3   & autoMPG8    & mpg & 392        & 24           & 22.5         & 6.524      & 0.497            & 0.497            & 29                  & 29                  & 4.791                  & 4.791               & 1              & 96               & 67               & $0.11$              & $\mathbf{0.09}$     \\
4   & baseball    & salary & 337        & 24           & 740          & 954.386    & $\mathbf{0.362}$ & 0.003            & 1550                & $\mathbf{2500}$     & $\underline{1245.092}$ & $\mathbf{0}$        & 1              & 117              & 117              & $0.22$              & $\mathbf{0.21}$     \\
5   & california  & med. h. value & $20,640$   & 72           & $179,700$    & $88,354$   & $\mathbf{0.385}$ & $0.019$          & $262,500$           & $\mathbf{500,001}$  & $\underline{94261}$    & $\mathbf{294,00}$   & 0.4            & $1,368,662$      & $65,707$         & $2,676$             & $\mathbf{368}$      \\
6   & compactiv   & usr & 8,192      & 202          & 89           & 9.661      & 0.464            & $\mathbf{0.603}$ & $\mathbf{94}$       & 93                  & 7.8                    & $\mathbf{3.472}$    & 0.5            & $2,458,105$      & $59,053$         & $5,161$             & $\mathbf{208}$      \\
7   & concrete    & compr. strength & 1,030      & 70           & 34.4         & 13.427     & $\mathbf{0.284}$ & 0.1291           & 48.97               & $\mathbf{50.7}$     & 12.744                 & $\mathbf{9.512}$    & 1              & $512,195$        & $221,322$        & $43.9$              & $\mathbf{35.8}$     \\
8   & dee         & consume & 365        & 60           & 2.787        & 0.831      & $\mathbf{0.523}$ & $0.381$          & $3.815$             & $\mathbf{4.008}$    & $0.721$                & $\mathbf{0.434}$    & 1              & $18,663$         & $2,653$          & $2.05$              & $\mathbf{1.29}$     \\
9   & delta\_ail  & sa & $7,129$    & 66           & $-0.0001$    & $0.000231$ & $\mathbf{0.902}$ & 0.392            & $0.0001$            & $\mathbf{0.0002}$   & $0.000226$             & $\mathbf{0.000119}$ & 1              & $45,194$         & $2,632$          & $33.3$              & $\mathbf{6.11}$     \\
10  & delta\_elv  & se & $9,517$    & 66           & 0.001        & 0.00198    & $\mathbf{0.384}$ & 0.369            & 0.002               & $0.002$             & 0.00112                & $\mathbf{0.00108}$  & 1              & 10145            & $1,415$          & $8.9$               & $\mathbf{4.01}$     \\
11  & elevators   & goal & $16,599$   & 155          & 0.02         & 0.00411    & 0.113            & $\mathbf{0.283}$ & $\mathbf{0.03}$     & $0.021$             & $\underline{0.00813}$  & $\mathbf{0.00373}$  & 0.05           & $6,356,465$      & $526,114$        & $13,712$            & $\mathbf{2,891}$    \\
12  & forestfires & area & 517        & 70           & 0.52         & 12.832     & $\mathbf{0.01}$  & $0.002$          & $86.45$             & $\mathbf{278.53}$   & $\underline{56.027}$   & $\mathbf{0}$        & 1              & $340,426$        & $264,207$        & $\mathbf{23}$       & $23.7$              \\
13  & friedman    & output & $1,200$    & 48           & 14.651       & 4.234      & $\mathbf{0.387}$ & 0.294            & 18.934              & $\mathbf{19.727}$   & 3.065                  & $\mathbf{2.73}$     & 1              & $19,209$         & $2,489$          & $3.23$              & $\mathbf{1.56}$     \\
14  & house       & price & $22,784$   & 160          & $33,200$     & 28,456     & 0.56             & $\mathbf{0.723}$ & $\mathbf{45,200}$   & $34,000$            & $\underline{40,576}$   & $\mathbf{27,214}$   & 0.002          & $1,221,696$      & $114,566$        & $7,937$             & $\mathbf{1,308}$    \\
15  & laser       & output & $993$      & 42           & 46           & 35.561     & $\mathbf{0.32}$  & 0.093            & 109                 & $\mathbf{135}$      & $\underline{40.313}$   & $\mathbf{15.662}$   & 1              & $2,008$          & $815$            & $0.96$              & $\mathbf{0.83}$     \\
16  & mortgage    & 30 y. rate & $1,049$    & 128          & 6.71         & 2.373      & $\mathbf{0.256}$ & 0.097            & 11.61               & $\mathbf{14.41}$    & 2.081                  & $\mathbf{0.98}$     & 1              & $40,753$         & $1,270$          & $11.6$              & $\mathbf{1.59}$     \\
17  & mv          & y & $40,768$   & 79           & -5.02086     & 8.509      & $\mathbf{0.497}$ & $0.349$          & 0.076               & $\mathbf{0.193}$    & $\underline{8.541}$    & $\mathbf{2.032}$    & 1              & $6,513$          & $1,017$          & $31.9$              & $\mathbf{13.2}$     \\
18  & pole        & output & $14,998$   & 260          & 0            & 28.949     & $\mathbf{0.40}$  & 0.24             & 100                 & 100                 & $\underline{38.995}$   & $\mathbf{16.692}$   & 0.2            & $1,041,146$      & $2,966$          & $2,638$             & $\mathbf{15}$       \\
19  & puma32h     & thetadd6 & $8,192$    & 318          & $0.000261$   & 0.023      & $\mathbf{0.299}$ & 0.244            & 0.026               & $\mathbf{0.031}$    & 0.018                  & $\mathbf{0.017}$    & 0.4            & $3,141,046$      & $5,782$          & $2,648$             & $\mathbf{15.5}$     \\
20  & stock       & company10 & 950        & 80           & 46.625       & 5.47       & $\mathbf{0.471}$ & 0.337            & 52.5                & $\mathbf{54.375}$   & 3.741                  & $\mathbf{2.515}$    & 1              & $85,692$         & $1,822$          & $12.5$              & $\mathbf{1.56}$     \\
21  & treasury    & 1 m. def. rate & $1,049$    & 128          & 6.61         & 2.473      & 0.182            & $\mathbf{0.339}$ & $\mathbf{13.16}$    & $8.65$              & $\underline{2.591}$    & $\mathbf{0.863}$    & 1              & $49,197$         & $9,247$          & $14.8$              & $\mathbf{5.91}$     \\
22  & wankara     & mean temp. & $321$    & 87           & 47.7         & 12.753     & $\mathbf{0.545}$ & 0.296            & 60.6                & $\mathbf{67.6}$     & 8.873                  & $\mathbf{4.752}$    & 1              & $191,053$        & $4,081$          & $11.9$              & $\mathbf{1.24}$     \\
23  & wizmir      & mean temp. & $1,461$    & 82           & 60           & 12.622     & $\mathbf{0.6}$   & 0.349            & 72.9                & $\mathbf{78.5}$     & 8.527                  & $\mathbf{3.889}$    & 1              & $177,768$        & $1,409$          & $38.5$              & $\mathbf{1.48}$     \\ \midrule
24  & binaries    & delta E & 82         & 499          & 0.106        & 0.277      & $0.305$          & $\mathbf{0.378}$ & $\mathbf{0.43}$     & 0.202               & $\underline{0.373}$    & $\mathbf{0.118}$    & 0.5            & $4,712,128$      & $204$            & $1,200$             & $\mathbf{0.29}$     \\
25  & gold        & Evdw-Evdw0 & $12,200$   & 250          & 0.131        & 0.088      & $\mathbf{0.765}$ & 0.34             & 0.217               & $\mathbf{0.234}$    & 0.081                  & $\mathbf{0.0278}$   & 0.4            & $1,498,185$      & $451$            & $5,650$             & $\mathbf{3.96}$     \\
\bottomrule
\end{tabular}%
}
\caption{Datasets with corresponding population size ($\card{P}$), number of base propositions ($\card{\Pi}$), global median ($\median{P}$) and mean absolute median deviation (\aamd{P}) followed by coverage ($\cov{\bestwrm}$, $\cov{\bestdcm}$), median ($\median{\bestwrm}$, $\median{\bestdcm}$), and mean absolute median deviation ($\aamd{\bestwrm}$, $\aamd{\bestdcm}$) for best subgroup w.r.t.\ non-dispersion corrected function $f_0$ and dispersion-corrected function $f_1$, respectively; bold-face indicates higher coverage and median and lower dispersion, underlines indicate higher dispersion than in global population; final column segment contains accuracy parameter used in the efficiency study ($a_\text{eff}$) as well as number of expanded nodes ($\numnodesloose$, $\numnodestight$) and computation time in seconds ($t_0$, $t_1$) for optimistic estimator based on top sequence $\oest{f_0}$ and tight optimistic estimator $\oest{f_1}$, respectively---in both cases when optimizing $f_1$; depth-limit of $10$ is used for all datasets with $a<1$, no depth-limit otherwise.}
\label{tab:bias}
\end{sidewaystable*}
The overall result of Sec.~\ref{sec:tightoests} is an efficient algorithm for dispersion-corrected subgroup discovery which, e.g., allows us to replace the coverage term in standard objective functions by the dispersion-corrected coverage.
To evaluate this efficiency claim as well as the value of dispersion-correction, let us consider as objective the normalized and dispersion-corrected impact function based on the median, i.e., $f_1(Q)=\mdcc{Q} \npmds{Q}$ where $\npmdssymb$ is the \defemph{positive relative median shift} 
\[
\npmds{Q}=\left ( \frac{\median{Q}-\median{P}}{\maxi{P}-\median{P}} \right )_+ \enspace .
\]
This function obeys Eq.~\eqref{eq:dccobj}; thus, its tight optimistic estimator can be computed using the linear time algorithm from Sec.~\ref{sec:lineartime}.
The following empirical results were gathered by applying it to a range of publicly available real-world datasets.\!\footnote{Datasets contain all regression datasets from the KEEL repository \citep{alcala2010keel} with at least 5 attributes and two materials datasets from the Nomad Repository \url{nomad-coe.eu/}; see Tab.~\ref{tab:bias}. Implementation available in open source Java library realKD \url{bitbucket.org/realKD/}. Computation times determined on MacBook Pro 3.1 GHz Intel Core i7.}
We will first investigate the effect of dispersion-correction on the output before turning to the effect of the tight optimistic estimator on the computation time. 

\subsection{Selection Bias of Dispersion-Correction and its Statistical Merit}
\label{sec:selection_bias}
To investigate the selection bias of $f_1$ let us also consider the non-dispersion corrected variant $f_0(Q)=\cov{Q} \npmds{Q}$ where we simply replace the dispersion-corrected coverage by the ordinary coverage. This function is a monotone level 1 function, hence, its tight optimistic estimator $\oest{f_0}$ can be computed in linear time using the top sequence approach.
Fig.~\ref{fig:selection_bias} shows the characteristics of the optimal subgroups that are discovered with respect to both of these objective functions (see also Tab.~\ref{tab:bias} for exact values) where for all datasets the language of closed conjunctions $\cnjclosedlang$ has been used as description language.

The first observation is that---as enforced by design---for all datasets the mean absolute deviation from the median is lower for the dispersion-corrected variant (except in one case where both functions yield the same subgroup). On average the dispersion for $f_1$ is $49$ percent of the global dispersion, whereas it is $113$ percent for $f_0$, i.e., \emph{when not optimizing the dispersion it is on average higher in the subgroups than in the global population}.
When it comes to the other subgroup characteristics, coverage and median target value, the global picture is that $f_1$ discovers somewhat more specific groups (mean coverage $0.3$ versus $0.44$ for $f_0$) with higher median shift (on average $0.73$ normalized median deviations higher). 
However, in contrast to dispersion, the behavior for median shift and coverage varies across the datasets.
In Fig.~\ref{fig:selection_bias}, the datasets are ordered according to the difference in subgroup medians between the optimal subgroups w.r.t.\ $f_0$ and those w.r.t.\ $f_1$.
This ordering reveals the following categorization of outcomes: When our description language is not able to reduce the error of subgroups with very high median value, $f_1$ settles for more coherent groups with a less extreme but still outstanding central tendency.
On the other end of the scale, when no coherent groups with moderate size and median shift can be identified, the dispersion-corrected objective selects very small groups with the most extreme target values.
The majority of datasets obey the global trend of dispersion-correction leading to somewhat more specific subgroups with higher median that are, as intended, more coherent.

\newcommand{\consmeansymb}{l}
\newcommand{\consmean}[1]{\consmeansymb(#1)}
\newcommand{\stdconsmeansymb}{\tilde{l}}
\newcommand{\stdconsmean}[1]{\stdconsmeansymb(#1)}
\begin{figure}[t]
\centering
\subfigure{
\includegraphics[width=\pairedimagewidth]{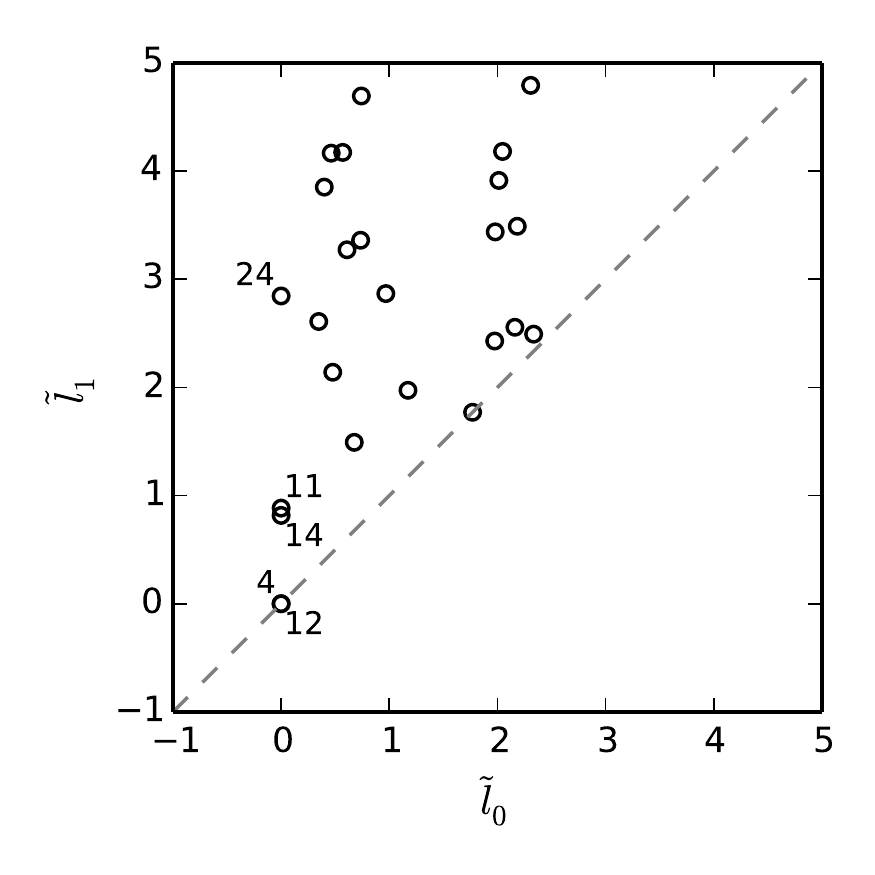}
}
\subfigure{
\includegraphics[width=\pairedimagewidth, trim={2.5cm 0 1.2cm 2cm},clip]{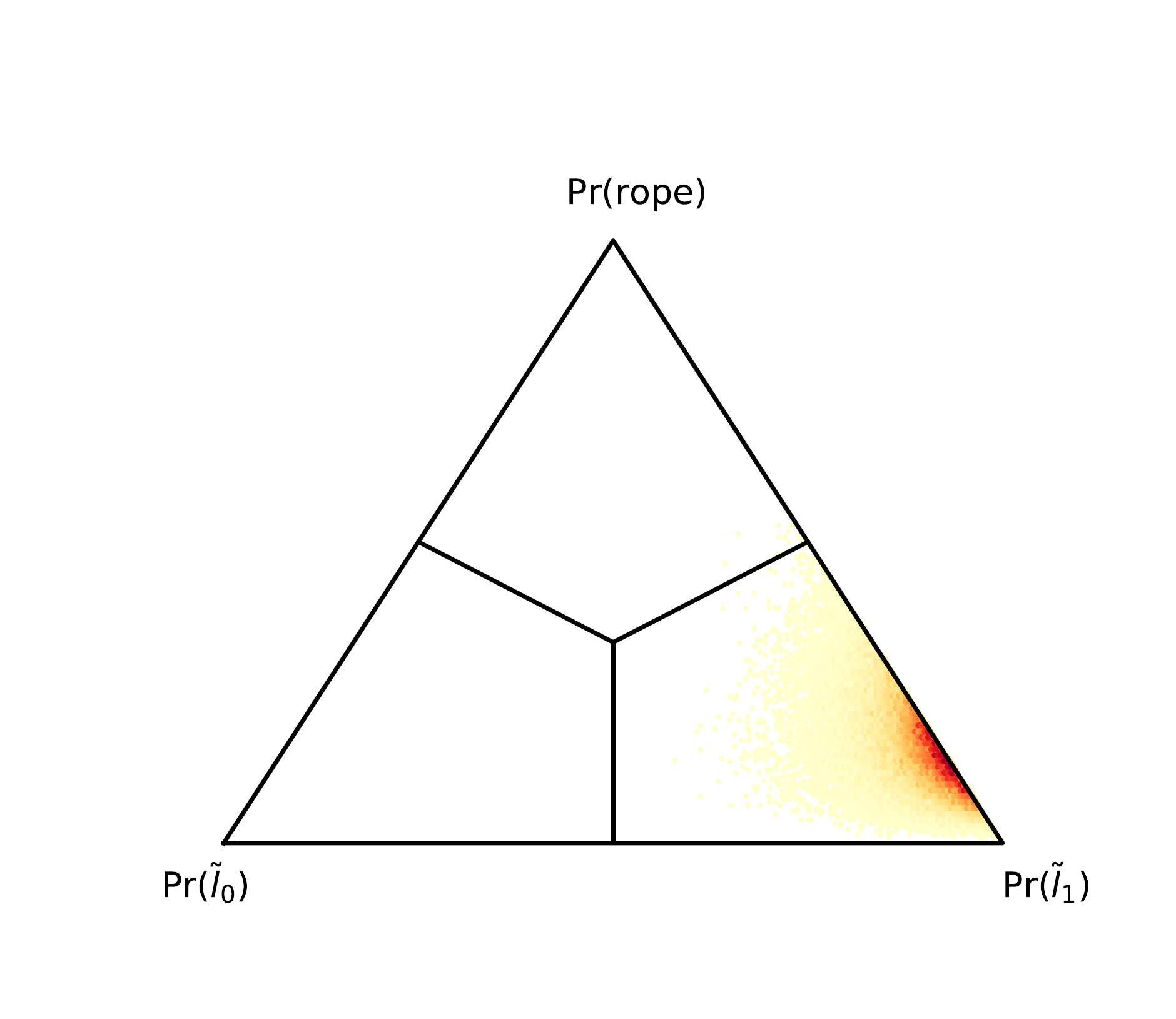}
}
\vspace{-0.5cm}
\caption{Effect of dispersion correction on lower bound of 95-percent confidence interval of target variable; \textbf{(left)} improvement over global lower bound in standard deviations of dispersion-corrected objective ($\tilde{l}_1$) and non-dispersion-corrected objective ($\tilde{l}_0$) with annotations showing ids of datasets where either method provides no improvement;
\textbf{(right)} posterior joint probabilities of the events that normalized difference $(\tilde{l}_1-\tilde{l}_0)/\max \{\tilde{l}_0,\tilde{l}_1\}$ 
is larger than $0.1$ ($\Pr(\tilde{l}_1)$), less than $-0.1$ ($\Pr(\tilde{l}_0)$), or within $[-0.1,0.1]$ ($\Pr(\mathrm{rope})$) according to Bayesian sign-test in barycentric coordinates (sections correspond to regions where corresponding event is maximum a posteriori outcome).}
\label{fig:lower_conf}
\end{figure}
To determine based on these empirical observations, whether we should generally favor dispersion correction, we have to specify an application context that specifies the relative importance of coverage, central tendency, and dispersion. For that let us consider the common statistical setting in which we do not observe the full global population $P$ but instead subgroup discovery is performed only on an i.i.d. sample $P' \subseteq P$ yielding subpopulations $Q'=\sigma(P')$. While $\sigma$ has been optimized w.r.t. the statistics on that sample $Q'$ we are actually interested in the properties of the full subpopulation $Q = \sigma(P)$.
For instance, a natural question is what is the minimal $y$-value that we expect to see in a random individual $q \in Q$ with high confidence.
That is, we prefer subgroups with an as high as possible threshold $\consmeansymb$ such that a random $q \in Q$ satisfies with probability\footnote{The probability is w.r.t. to the  distribution with which the sample $P' \subseteq P$ is drawn.} $1-\delta$ that $y(q) \geq \consmeansymb$.
This criterion gives rise to a natural trade-off between the three evaluation metrics through the \defemph{empirical Chebycheff inequality} \citep[see][Eq.~(17)]{kaban2012non}, according to which we can compute such a value as $\average{Q'}-\epsilon(Q')$ where
\[
\epsilon(Q') = \sqrt{\frac{(\card{Q'}^2-1)\texttt{var}(Q')}{\card{Q'}^2\delta - \card{Q'}}}
\]
and $\variance{Y}=\sum_{y \in Y} (y-\average{Y})^2 / (\card{Y}-1)$ is the sample variance. Note that this expression is only defined for sample subpopulations with a size of at least $1/\delta$. For smaller subgroups our best guess for a threshold value would be the one derived from the global sample $\average{P'}-\epsilon(P')$ (which we assume to be large enough to determine an $\epsilon$-value). This gives rise to the following \defemph{standardized lower confidence bound score} $\stdconsmeansymb$ that evaluates how much a subgroup improves over the global $\consmeansymb$ value:
\[
\stdconsmean{Q'}=\left( \frac{\consmean{Q'}-\consmean{P'}}{\sqrt{\texttt{var}(P')}} \right )_+ \enspace \text{where} \enspace
\consmean{Q'}=\begin{cases}
\average{Q'}-\epsilon(Q') & \text{, if } \epsilon(Q') \text{ defined}\\
\average{P'}-\epsilon(P') & \text{, otherwise}
\end{cases}
\enspace .
\]

The plot on the left side of Fig.~\ref{fig:lower_conf} shows the score values of the optimal subgroup w.r.t. to $f_1$ ($\stdconsmeansymb_1$) and $f_0$ ($\stdconsmeansymb_0$) using confidence parameter $\delta=0.05$. Except for three exceptions (datasets 3,4, and 12), the subgroup resulting from $f_1$ provides a higher lower bound than those from the non-dispersion corrected variant $f_0$. That is, the data shows a strong advantage for dispersion correction when we are interested in selectors that mostly select individuals with a high target value from the underlying population $P$.
In order to test the significance of these results, we can employ the \defemph{Bayesian sign-test} \citep{benavoli2014bayesian}, a modern alternative to classic frequentist null hypothesis tests that avoids many of the well-known disadvantages of those \citep[see][]{demvsar2008appropriateness, benavoli2016time}. With Bayesian hypothesis tests, we can directly evaluate the posterior probabilities of hypotheses given our experimental data instead of just rejecting a null hypothesis based on some arbitrary significance level. Moreover, we differentiate between sample size and effect size by the introduction of a region of practical equivalence (rope). Here, we are interested in the relative difference $\tilde{z}=(\stdconsmeansymb_1 - \stdconsmeansymb_0)/(\max \{\stdconsmeansymb_0, \stdconsmeansymb_1\})$ on average for random subgroup discovery problems.
Using a conservative choice for the rope, we call the two objective functions practically equivalent if the mean $\tilde{z}$-value is at most $r=0.1$. Choosing the prior belief that $f_0$ is superior, i.e., $\tilde{z} < -r$, with a prior weight of $1$, the procedure yields based on our 25 test datasets the posterior probability of approximately $1$ that $\tilde{z} > r$ on average (see the right part of Fig.~\ref{fig:lower_conf} for in illustration of the posterior belief). Hence, we can conclude that dispersion-correction improves the relative lower confidence bound of target values on average by more than 10 percent when compared to the non-dispersion-corrected function.

\begin{figure}[t]
\centering
\subfigure{
\includegraphics[width=\pairedimagewidth]{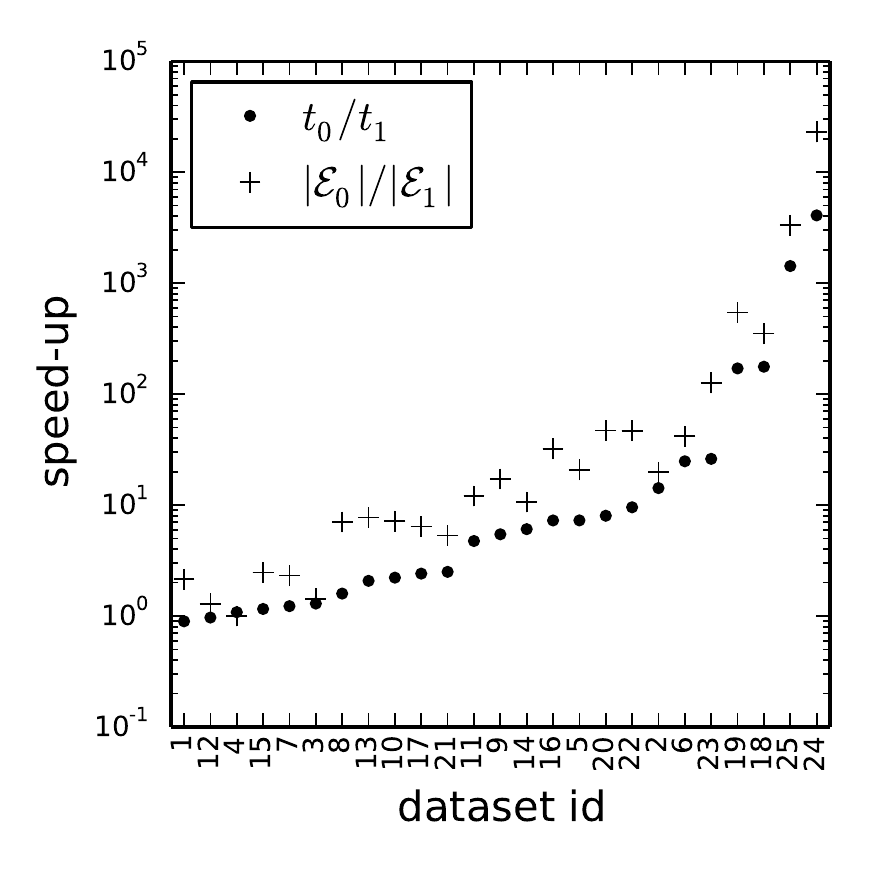}
}
\subfigure{
\includegraphics[width=\pairedimagewidth]{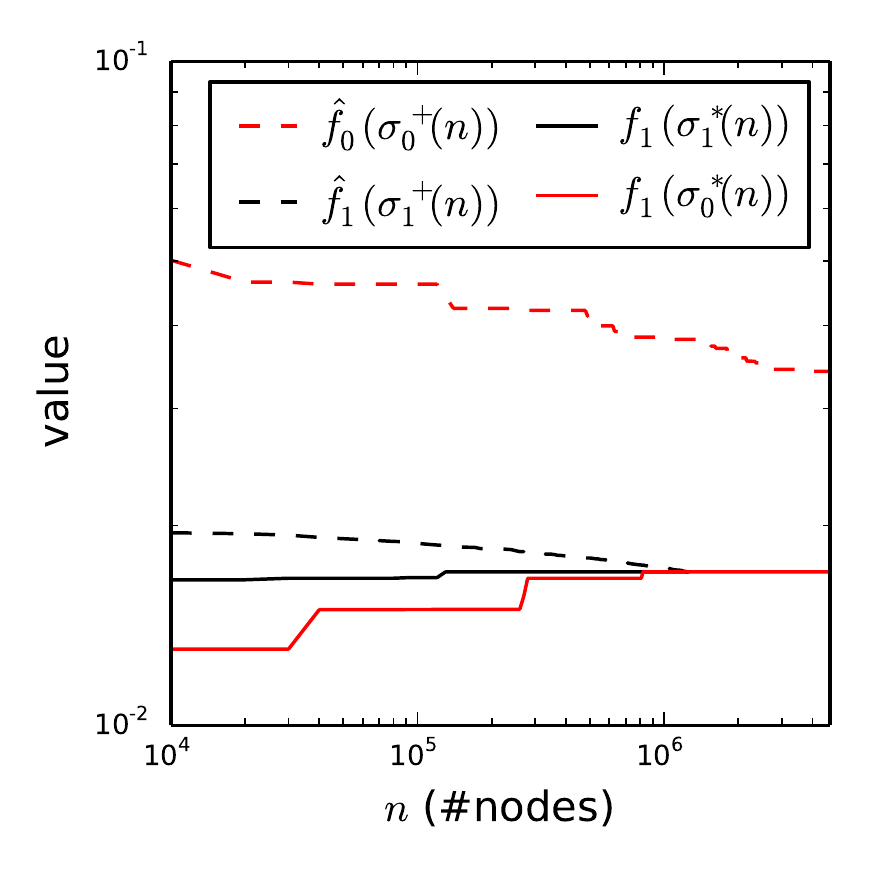}
}
\vspace{-0.5cm}
\caption{Effect of tight optimistic estimator; \textbf{(right)} optimistic estimation ($\oest{f_1}(\sigma^+_1)$, $\oest{f_0}(\sigma^+_0)$) of remaining search space and value of current best solution ($f_1(\sigma^*_1)$, $f_1(\sigma^*_0)$) resulting from tight estimator and top sequence estimator, respectively---per processed nodes for dataset \textit{binaries};
\textbf{(left)} speedup factor ($t_0/t_1$) in order for all datasets plus potential speed-up factor ($\numnodesloose/\numnodestight$).}
\label{fig:tight_effect}
\end{figure}
\subsection{Efficiency of the Tight Optimistic Estimator}
To study the effect of the tight optimistic estimator, let us compare its performance to that of a baseline estimator that can be computed with the standard top sequence approach. Since $f_1$ is upper bounded by $f_0$, $\oest{f_0}$ is a valid, albeit non-tight, optimistic estimator for $f_1$ and can thus be used for this purpose.
The exact speed-up factor is determined by the ratio of enumerated nodes for both variants as well as the ratio of computation times for an individual optimistic estimator computation.
While both factors determine the practically relevant outcome, the number of nodes evaluated is a much more stable quantity, which indicates the full underlying speed-up potential independent of implementation details. 
Similarly, ``number of nodes evaluated'' is also an insightful unit of time for measuring optimization progress.
Therefore, in addition to the computation time in seconds $t_0$ and $t_1$, let us denote by $\nodesloose, \nodestight \subseteq \cL$ the set of nodes enumerated by branch-and-bound using $\oest{f_0}$ and $\oest{f_1}$, respectively---\emph{but in both cases for optimizing the dispersion-corrected objective} $f_1$.
Moreover, when running branch-and-bound with optimistic estimator $\oest{f_i}$, let us denote by $\sigma^*_i(n)$ and $\sigma^+_i(n)$ the best selector found and the top element of the priority queue (w.r.t.\ $\oest{f_i}$), respectively, after $n$ nodes have been enumerated.

The plot on the left side of Fig.~\ref{fig:tight_effect} shows the speed-up factor $t_1/t_0$ on a logarithmic axis for all datasets in increasing order along with the potential speed-up factors $\numnodesloose/\numnodestight$ (see Tab.~\ref{tab:bias} for numerical values). There are seven datasets for which the speed-up turns out to be minor followed by four datasets with a modest speed-up factor of 2. For the remaining 14 datasets, however, we have solid speed-up factors between $4$ and $20$ and in four cases immense values between $100$ and $4,000$.
This demonstrates the decisive potential effect of tight value estimation even when compared to another non-trivial estimator like $\oest{f_0}$ (which itself improves over simpler options by orders of magnitude; see \citealt{lemmerich2016fast}). 
Similar to the results in Sec.~\ref{sec:selection_bias}, the Bayesian sign-test for the normalized difference $z=(t_1-t_0)/\max\{t_1,t_0\}$ with the prior set to practical equivalence ($z \in [-0.1,0.1]$) reveals that the posterior probability of $\oest{f}_1$ being superior to $\oest{f}_0$ is approximately 1.

In almost all cases the potential speed-up given by the ratio of enumerated nodes is considerably higher than the actual speed-up, which shows
that, despite the same asymptotic time complexity, an individual computation of the tight optimistic estimator is slower than the simpler top sequence based estimator---but also indicates that there is room for improvements in the implementation.
When zooming in on the optimization progress over time for the \textit{binaries} dataset, which exhibits the most extreme speed-up (right plot in Fig.~\ref{fig:tight_effect}),
we can see that not only does the tight optimistic estimator close the gap between best current selector and current highest potential selector much faster---thus creating the huge speed-up factor---but also that it causes better solutions to be found earlier.
This is an important property when we want to use the algorithm as an \emph{anytime algorithm}, i.e., when allowing the user to terminate computation preemptively, which is important in interactive data analysis systems.
This is an advantage enabled specifically by using the tight optimistic estimators in conjunction with the best-first node expansion strategy.

\section{Conclusion}
During the preceding sections, we developed and evaluated an effective algorithm for simultaneously optimizing size, central tendency, and dispersion in subgroup discovery with a numerical target.
This algorithm is based on two central results: 1) the tight optimistic estimator for any objective function that is based on some dispersion measure around the median can be computed as the function's maximum on a linear-sized sequence of sets---the median sequence (Prop.~\ref{prop:medianseq}); and 2) for objective functions based on the concept of the dispersion-corrected coverage w.r.t. the absolute deviation from the median, the individual sets of the median sequence can be generated in incremental constant time (Thm.~\ref{thm:to7}).

\emph{Among the possible applications of the proposed approach}, the perhaps most important one is to replace the standard coverage term in classic objective functions by the dispersion-corrected coverage, i.e., the relative subgroup size minus the relative subgroup dispersion, to reduce the error of result subgroups---where error refers to the descriptive or predictive inaccuracy incurred when assuming the median value of a subgroup for all its members.
As we saw empirically for the impact function (based on the median), this correction also has a statistical advantage resulting in subgroups where we can assume greater target values for unseen group members with high confidence.
In addition to enabling dispersion-correction to known objective functions, the presented algorithm also provides novel degrees of freedom, which might be interesting to exploit in their own right:
The dependence on the median is not required to be monotone, which allows to incorporate a more sophisticated influence of the central tendency value than simple monotone average shifts. For instance, given a suitable statistical model for the global distribution, the effect of the median could be a function of the probability $\mathbb{P}[\median{Q}]$, e.g., its Shannon information content.
Furthermore, the feasible dispersion measures allow for interesting weighting schemes, which include possibilities of asymmetric effects of the error (e.g., for only punishing one-sided deviation from the median).

\emph{Regarding the limitations of the presented approach}, let us note that it cannot be directly applied to the previously proposed dispersion-aware functions, i.e., the $t$-score $\mathtt{tsc}(Q)=\sqrt{\card{Q}}(\average{Q}-\average{P})/\std{Q}$ and the mmad score for ranked data ${\mathtt{mmd}}(Q)=\card{Q}/(2\median{Q}+\mamd{Q})$. 
While both of these functions can be optimized via the median sequence approach (assuming a $t$-score variant based on the median), we are lacking an efficient incremental formula for computing the individual function values for all median sequence sets, i.e., a replacement for Thm.~\ref{thm:to7}.
Though finding such a replacement in future research is conceivable, 
this leaves us for the moment with a quadratic time algorithm (in the subgroup size) for the tight optimistic estimator, which is not generally feasible (although potentially useful for smaller datasets or as part of a hybrid optimistic estimator, which uses the approach for sufficiently small subgroups only).

Since they share basic monotonicities, it is possible to use functions based on dispersion-corrected coverage as an optimization proxy for the above mentioned objectives. For instance, the ranking of the top 20 subgroups w.r.t. the dispersion-corrected binomial quality function, $\mathtt{dcb}(Q)=\sqrt{\mdcc{Q}}(\median{Q}-\median{P})$, turns out to have a mean Spearman rank correlation coefficient with the median-based $t$-score of apx. $0.783$ on five randomly selected test datasets (\textit{delta\_elv}, \textit{laser}, \textit{stock}, \textit{treasury}, \textit{gold}). However, a more systematic understanding of the differences and commonalities of these functions is necessary to reliably replace them with one another. Moreover, the correlation deteriorates quite sharply when we compare to the original mean/variance based $t$-score (mean Spearman correlation coefficient 0.567), which points to the perhaps more fundamental limitation of the presented approach for dispersion-correction: it relies on using the median as measure of central tendency.
While the median and the mean absolute deviation from the median are an interpretable, robust, and sound combination of measures (the median of a set of values minimizes the sum of absolute deviations), the mean and the variance are just as sound, are potentially more relevant when sensitivity to outliers is required, and provide a wealth of statistical tools (e.g., the empirical Chebyshev's inequality used above).

\emph{Hence, a straightforward but valuable direction for future work} is the extension of efficient tight optimistic estimator computation to dispersion-correction based on the mean and variance. A basic observation for this task is that objective functions based on dispersion measures around the mean must also attain their maximum on gap-free intervals of target values. However, for a given collection of target values, there is a quadratic number of intervals such that a further idea is required in order to attain an efficient, i.e., (log-)linear time algorithm.
Another valuable direction for future research is the extension of consistency and error optimization to the case of multidimensional target variables where subgroup parameters can represent complex statistical models \citep[known as \textit{exceptional model mining}][]{duivesteijn2016exceptional}. While this setting is algorithmically more challenging than the univariate case covered here, the underlying motivation remains: balancing group size and exceptionality, i.e., distance of local to global model parameters, with consistency, i.e., local model fit, should lead to the discovery of more meaningful statements about the data and the underlying domain.

\paragraph{Acknowledgements}
The authors thank the anonymous reviewers for their useful and constructive suggestions.
Jilles Vreeken and Mario Boley are supported by the Cluster of Excellence ``Multimodal Computing and Interaction'' within the Excellence Initiative of the German Federal Government. 
Bryan R. Goldsmith acknowledges support from the Alexander von Humboldt-Foundation with a Postdoctoral Fellowship.
Additionally, this work was supported through the European Union's Horizon 2020 research and innovation program under
grant agreement no. 676580 with The Novel Materials Discovery (NOMAD) Laboratory, a European Center of
Excellence.

\bibliographystyle{spbasic} 
\bibliography{biblio}

\begin{thebibliography}{31}
\providecommand{\natexlab}[1]{#1}
\providecommand{\url}[1]{{#1}}
\providecommand{\urlprefix}{URL }
\expandafter\ifx\csname urlstyle\endcsname\relax
  \providecommand{\doi}[1]{DOI~\discretionary{}{}{}#1}\else
  \providecommand{\doi}{DOI~\discretionary{}{}{}\begingroup
  \urlstyle{rm}\Url}\fi
\providecommand{\eprint}[2][]{\url{#2}}

\bibitem[{Alcal{\'a} et~al(2010)Alcal{\'a}, Fern{\'a}ndez, Luengo, Derrac,
  Garc{\'\i}a, S{\'a}nchez, and Herrera}]{alcala2010keel}
Alcal{\'a} J, Fern{\'a}ndez A, Luengo J, Derrac J, Garc{\'\i}a S, S{\'a}nchez
  L, Herrera F (2010) Keel data-mining software tool: Data set repository,
  integration of algorithms and experimental analysis framework. Journal of
  Multiple-Valued Logic and Soft Computing 17(2-3):255--287

\bibitem[{Atzmueller(2015)}]{atzmueller2015subgroup}
Atzmueller M (2015) Subgroup discovery. Wiley Interdisciplinary Reviews: Data
  Mining and Knowledge Discovery 5(1):35--49

\bibitem[{Bay and Pazzani(2001)}]{bay2001detecting}
Bay SD, Pazzani MJ (2001) Detecting group differences: Mining contrast sets.
  Data Mining and Knowledge Discovery 5(3):213--246

\bibitem[{Benavoli et~al(2014)Benavoli, Corani, Mangili, Zaffalon, and
  Ruggeri}]{benavoli2014bayesian}
Benavoli A, Corani G, Mangili F, Zaffalon M, Ruggeri F (2014) A bayesian
  wilcoxon signed-rank test based on the dirichlet process. In: ICML, pp
  1026--1034

\bibitem[{Benavoli et~al(2016)Benavoli, Corani, Demsar, and
  Zaffalon}]{benavoli2016time}
Benavoli A, Corani G, Demsar J, Zaffalon M (2016) Time for a change: a tutorial
  for comparing multiple classifiers through bayesian analysis. arXiv preprint
  arXiv:160604316

\bibitem[{Boley and Grosskreutz(2009)}]{boley2009non}
Boley M, Grosskreutz H (2009) Non-redundant subgroup discovery using a closure
  system. In: Joint European Conf. on Machine Learning and Knowledge Discovery
  in Databases, Springer, pp 179--194

\bibitem[{Boley et~al(2012)Boley, Moens, and G{\"a}rtner}]{boley2012linear}
Boley M, Moens S, G{\"a}rtner T (2012) Linear space direct pattern sampling
  using coupling from the past. In: Proc. of the 18th ACM SIGKDD int. conf. on
  Knowledge discovery and data mining, ACM, pp 69--77

\bibitem[{Dem{\v{s}}ar(2008)}]{demvsar2008appropriateness}
Dem{\v{s}}ar J (2008) On the appropriateness of statistical tests in machine
  learning. In: Workshop on Evaluation Methods for Machine Learning in
  conjunction with ICML

\bibitem[{Duivesteijn and Knobbe(2011)}]{duivesteijn2011exploiting}
Duivesteijn W, Knobbe A (2011) Exploiting false discoveries--statistical
  validation of patterns and quality measures in subgroup discovery. In: IEEE
  11th Int. Conf. on Data Mining, IEEE, pp 151--160

\bibitem[{Duivesteijn et~al(2016)Duivesteijn, Feelders, and
  Knobbe}]{duivesteijn2016exceptional}
Duivesteijn W, Feelders AJ, Knobbe A (2016) Exceptional model mining. Data
  Mining and Knowledge Discovery 30(1):47--98

\bibitem[{Friedman and Fisher(1999)}]{friedman1999bump}
Friedman JH, Fisher NI (1999) Bump hunting in high-dimensional data. Statistics
  and Computing 9(2):123--143

\bibitem[{Goldsmith et~al(2017)Goldsmith, Boley, Vreeken, Scheffler, and
  Ghiringhelli}]{goldsmith2017uncovering}
Goldsmith BR, Boley M, Vreeken J, Scheffler M, Ghiringhelli LM (2017)
  Uncovering structure-property relationships of materials by subgroup
  discovery. New Journal of Physics 19(1):13--31

\bibitem[{Grosskreutz et~al(2008)Grosskreutz, R{\"u}ping, and
  Wrobel}]{grosskreutz2008tight}
Grosskreutz H, R{\"u}ping S, Wrobel S (2008) Tight optimistic estimates for
  fast subgroup discovery. In: Joint European Conf. on Machine Learning and
  Knowledge Discovery in Databases, Springer, pp 440--456

\bibitem[{Grosskreutz et~al(2010)Grosskreutz, Boley, and
  Krause-Traudes}]{grosskreutz2010subgroup}
Grosskreutz H, Boley M, Krause-Traudes M (2010) Subgroup discovery for election
  analysis: a case study in descriptive data mining. In: Int. Conf. on
  Discovery Science, Springer, pp 57--71

\bibitem[{Huan et~al(2003)Huan, Wang, and Prins}]{huan2003efficient}
Huan J, Wang W, Prins J (2003) Efficient mining of frequent subgraphs in the
  presence of isomorphism. In: 3rd IEEE Int. Conf. on Data Mining, IEEE, pp
  549--552

\bibitem[{Kab{\'a}n(2012)}]{kaban2012non}
Kab{\'a}n A (2012) Non-parametric detection of meaningless distances in high
  dimensional data. Statistics and Computing 22(2):375--385

\bibitem[{Kl{\"o}sgen(1996)}]{klosgen1996explora}
Kl{\"o}sgen W (1996) Explora: A multipattern and multistrategy discovery
  assistant. In: Advances in knowledge discovery and data mining, American
  Association for Artificial Intelligence, pp 249--271

\bibitem[{Kl{\"o}sgen(2002)}]{klosgen2002data}
Kl{\"o}sgen W (2002) Data mining tasks and methods: Subgroup discovery:
  deviation analysis. In: Handbook of data mining and knowledge discovery,
  Oxford University Press, Inc., pp 354--361

\bibitem[{Lavra{\v{c}} et~al(2004)Lavra{\v{c}}, Kav{\v{s}}ek, Flach, and
  Todorovski}]{lavravc2004subgroup}
Lavra{\v{c}} N, Kav{\v{s}}ek B, Flach P, Todorovski L (2004) Subgroup discovery
  with cn2-sd. Journal of Machine Learning Research 5(Feb):153--188

\bibitem[{Lemmerich et~al(2016)Lemmerich, Atzmueller, and
  Puppe}]{lemmerich2016fast}
Lemmerich F, Atzmueller M, Puppe F (2016) Fast exhaustive subgroup discovery
  with numerical target concepts. Data Mining and Knowledge Discovery
  30(3):711--762

\bibitem[{Li and Zaki(2016)}]{li2016sampling}
Li G, Zaki MJ (2016) Sampling frequent and minimal boolean patterns: theory and
  application in classification. Data Mining and Knowledge Discovery
  30(1):181--225

\bibitem[{Mehlhorn and Sanders(2008)}]{mehlhorn2008algorithms}
Mehlhorn K, Sanders P (2008) Algorithms and data structures: The basic toolbox.
  Springer Science \& Business Media

\bibitem[{Parthasarathy et~al(1999)Parthasarathy, Zaki, Ogihara, and
  Dwarkadas}]{parthasarathy1999incremental}
Parthasarathy S, Zaki MJ, Ogihara M, Dwarkadas S (1999) Incremental and
  interactive sequence mining. In: Proc. 8th int. conf. on Information and
  knowledge management, ACM, pp 251--258

\bibitem[{Pasquier et~al(1999)Pasquier, Bastide, Taouil, and
  Lakhal}]{pasquier1999efficient}
Pasquier N, Bastide Y, Taouil R, Lakhal L (1999) Efficient mining of
  association rules using closed itemset lattices. Information systems
  24(1):25--46

\bibitem[{Pieters et~al(2010)Pieters, Knobbe, and
  Dzeroski}]{pieters2010subgroup}
Pieters BF, Knobbe A, Dzeroski S (2010) Subgroup discovery in ranked data, with
  an application to gene set enrichment. In: Proc. preference learning workshop
  (PL 2010) at ECML PKDD, vol~10, pp 1--18

\bibitem[{Schmidt et~al(2010)Schmidt, Hapfelmeier, Mueller, Perneczky, Kurz,
  Drzezga, and Kramer}]{schmidt2010interpreting}
Schmidt J, Hapfelmeier A, Mueller M, Perneczky R, Kurz A, Drzezga A, Kramer S
  (2010) Interpreting pet scans by structured patient data: a data mining case
  study in dementia research. Knowledge and Information Systems 24(1):149--170

\bibitem[{Song et~al(2016)Song, Kull, Flach, and Kalogridis}]{song2016subgroup}
Song H, Kull M, Flach P, Kalogridis G (2016) Subgroup discovery with proper
  scoring rules. In: Joint European Conference on Machine Learning and
  Knowledge Discovery in Databases, Springer, pp 492--510

\bibitem[{Uno et~al(2004)Uno, Asai, Uchida, and Arimura}]{uno2004efficient}
Uno T, Asai T, Uchida Y, Arimura H (2004) An efficient algorithm for
  enumerating closed patterns in transaction databases. In: Int. Conf. on
  Discovery Science, Springer, pp 16--31

\bibitem[{Webb(1995)}]{webb1995opus}
Webb GI (1995) Opus: An efficient admissible algorithm for unordered search.
  Journal of Artificial Intelligence Research 3:431--465

\bibitem[{Webb(2001)}]{webb2001discovering}
Webb GI (2001) Discovering associations with numeric variables. In: Proc. of
  the 7th ACM SIGKDD int. conf. on Knowledge discovery and data mining, ACM, pp
  383--388

\bibitem[{Wrobel(1997)}]{wrobel1997algorithm}
Wrobel S (1997) An algorithm for multi-relational discovery of subgroups. In:
  European Symposium on Principles of Data Mining and Knowledge Discovery,
  Springer, pp 78--87

\end{thebibliography}
\appendix
\section{Proof of Theorem~\ref{thm:to7}}
\label{apx:proofto7}
In order to proof Thm.~\ref{thm:to7}, let us start by noting that for functions of the form of Eq.~\eqref{eq:dccobj}, finding the set size $k^*_z$ corresponds to maximizing the dispersion-corrected coverage among all multisets with consecutive elements around median $y_z$ (as defined in Eq.~\ref{eq:consecset}).
In order to analyze this problem, let us write
\[
h_z(k)=\mdcc{Q^k_z}=\frac{\card{Q^k_z}}{\card{P}} - \frac{\sumae{Q^k_z}}{\sumae{P}}
\]
for the dispersion-corrected coverage of the multiset $Q^k_z$.
Let $\map{\Delta h_z}{\setupto{m_z}}{\R}$ denote the difference or gain function of $h_z$, i.e., $\Delta h_z(k)=h_z(k)-h_z(k-1)$ where we consider $Q^0_z=\emptyset$ and, hence, $h_z(0)=0$. With this definition we can show that $h_z$ is alternating between two concave functions, i.e., considering either only the even or only the odd subset of its domain, the gains are monotonically decreasing. More precisely:
\begin{lem}
\label{lm:altconcavity}
For all $k \in \setupto{m_z} \setminus \{1,2\}$ we have that $\Delta h_z(k) \leq \Delta h_z(k-2)$.
\end{lem}
\begin{proof}
For $k \in \setupto{m_z}$, let us denote by $q^k_z$ the additional $y$-value that $Q_z^k$ contains compared to $Q_z^{k-1}$ (considering $Q_z^{0}=\emptyset$), i.e., $Q_z^k \setminus Q_z^{k-1}=\{q_z^k\}$. We can check that 
\[
q_z^k=
\begin{cases}
q_{z-\left \lfloor \frac{k-1}{2}\right \rfloor}, &k \text{ odd}\\
q_{z+\left \lceil \frac{k-1}{2}\right \rceil}, &k \text{ even}
\end{cases} \enspace .
\]
With this and using the shorthands $n=\card{P}$ and $d=\sumae{P}$ we can write
\begin{align*}
&\Delta h_z(k) - \Delta h_z(k-2)
= h_z(k)-h_z(k-1)-(h_z(k-2)-h_z(k-3))\\
=& \frac{k}{n}-\frac{\sumae{Q^k_z}}{d}-\frac{k-1}{n}+\frac{\sumae{Q^{k-1}_z}}{d}-\frac{k-2}{n}+\frac{\sumae{Q^{k-2}_z}}{d}+\frac{k-3}{n}-\frac{\sumae{Q^{k-3}_z}}{d}\\
=&\frac{1}{n} \underbrace{\left (k-k+1-k+2+k-3 \right )}_{=0}+\frac{1}{d}\left (\sumae{Q_z^{k-2}}- \sumae{Q^k_z}+\sumae{Q_z^{k-1}}-\sumae{Q_z^{k-3}} \right )\\
=&\frac{1}{d} \left( -\abs{q_z^{k}-y_z} - \abs{q^{k-1}_z-y_z} + \abs{q_z^{k-1}-y_z} + \abs{q_z^{k-2}-y_z} \right)\\
=&\frac{1}{d} \left( -\abs{q_z^{k}-y_z} + \abs{q_z^{k-2}-y_z} \right)\\
\intertext{case $k$ odd}
=&\frac{1}{d} \left( -\left(y_z-y_{z-\left \lfloor \frac{k-1}{2}\right \rfloor}\right) + \left(y_z-y_{z-\left \lfloor \frac{k-3}{2}\right \rfloor}\right) \right) = y_{z-\left \lfloor \frac{k-1}{2}\right \rfloor} - y_{z-\left \lfloor \frac{k-3}{2}\right \rfloor} \leq 0\\
\intertext{case $k$ even}
=&\frac{1}{d} \left( -\left(y_{z+\left \lceil \frac{k-1}{2}\right \rceil}-y_z\right) + \left(y_{z+\left \lceil \frac{k-3}{2}\right \rceil}-y_z\right) \right) = y_{z+\left \lceil \frac{k-3}{2}\right \rceil} - y_{z+\left \lceil \frac{k-1}{2}\right \rceil} \leq 0
\end{align*}
\smartqed
\end{proof}
One important consequence of this fact is that the operation of growing a set around median $z$ by two elements---one to the left and one to the right---has monotonically decreasing gains.
In other words, the smoothed function $h_z(k)=h_z(k)+h_z(k-1)$ is concave or formally
\begin{equation}
\Delta h_z(k) + \Delta h_z(k-1) \geq \Delta h_z(k+1)+\Delta h_z(k) \enspace . 
\label{eq:smoothedconc}
\end{equation}
Moreover, we can relate the gain functions of consecutive median indices as follows.
\begin{lem}
\label{lm:link}
Let $z \in \setupto{m} \setminus \{1\}$ and $k \in \setupto{m_{z-1}}\setminus \{1,2,3\}$. It holds that
\begin{align}
\Delta h_{z-1}(k-2)+\Delta h_{z-1}(k-3) &\geq \Delta h_{z}(k)+\Delta h_{z}(k-1)\label{eq:linkupper}\\
\Delta h_{z-1}(k)+\Delta h_{z-1}(k-1) &\leq \Delta h_{z}(k-2)+\Delta h_{z}(k-3)\label{eq:linklower}
\end{align}
\end{lem}
\newcommand{\llceil}{\left\lceil}
\newcommand{\rrceil}{\right\rceil}
\newcommand{\llfloor}{\left\lfloor}
\newcommand{\rrfloor}{\right\rfloor}
\begin{proof}
For this proof, let us use the same shorthands as in the proof of Lemma~\ref{lm:altconcavity} and start by noting that for all $i \in \setupto{m}$ and $k \in \setupto{m_z} \setminus \{1\}$ we have the equality
\begin{equation}
\Delta h_{i}(k)+\Delta h_i(k-1) = \frac{2}{n}-\frac{\abs{q_i^k-y_i}+\abs{q_i^{k-1}-y_i}}{d}
\label{eq:prelim}
\end{equation}
which we can see by extending
\begin{align*}
\Delta h_{i}(k)+\Delta h_i(k-1) &=h_i(k)-h_i(k-1)+h_i(k-1)-h_i(k-2) \\
&=\frac{k-k+2}{n}-\frac{\sumae{Q^k_i}-\sumae{Q^{k-2}_i}}{d}=\frac{2}{n}-\frac{\abs{q^k_i-y_i}+\abs{q^{k-1}_i-y_i}}{d} \enspace .
\end{align*}
We can then show Eq.~\eqref{eq:linkupper} by applying Eq.~\eqref{eq:prelim} two times to
\begin{align*}
&\Delta h_{z-1}(k-2)+\Delta h_{z-1}(k-3) - (\Delta h_{z}(k)+\Delta h_{z}(k-1))\\
=& \frac{1}{d} \left ( -\abs{q^{k-2}_{z-1}-y_{z-1}} - \abs{q^{k-3}_{z-1}-y_{z-1}} + \abs{q^k_z-y_z} + \abs{q^{k-1}_z-y_z}\right )\\
\intertext{and finally by checking separately the case $k$ odd}
=& \frac{1}{d} \left ( y_{z-1-\llfloor \frac{k-3}{2} \rrfloor}-y_{z-1}+y_{z-1} - y_{z-1+\llceil \frac{k-4}{2} \rrceil} + y_z - y_{z-\llfloor \frac{k-1}{2}\rrfloor} + y_{z+\llceil \frac{k-2}{2}\right\rceil}-y_z\right )\\
=& \frac{1}{d} \left ( \underbrace{y_{z-\llfloor \frac{k-1}{2} \rrfloor} - y_{z-\llfloor \frac{k-1}{2}\rrfloor}}_{=0}  + \underbrace{y_{z-1+\llceil\frac{k}{2}\rrceil} - y_{z-3+\llceil \frac{k}{2} \rrceil}}_{\geq 0}\right ) \geq 0\\
\intertext{and the case $k$ even}
=& \frac{1}{d} \left(y_{z-1}- y_{z-1+\llceil \frac{k-3}{2} \rrceil} +  y_{z-1-\llfloor \frac{k-4}{2} \rrfloor}-y_{z-1} + y_{z+\llceil \frac{k-1}{2} \rrceil} - y_z + y_z - y_{z - \llfloor \frac{k-2}{2}\rrfloor} \right)\\
=& \frac{1}{d} \left(\underbrace{y_{z+1-\llfloor \frac{k}{2} \rrfloor} - y_{z +1 - \llfloor \frac{k}{2}\rrfloor}}_{=0} + \underbrace{y_{z+\llceil \frac{k-1}{2} \rrceil} - y_{z-2+\llceil \frac{k-1}{2} \rrceil}}_{\geq 0}  \right) \geq 0 \enspace .
\end{align*}
Similarly, for Eq.~\eqref{eq:linklower} by applying Eq.~\eqref{eq:prelim} two times we can write
\begin{align*}
&\Delta h_{z-1}(k)+\Delta h_{z-1}(k-1) - (\Delta h_z(k-2) + \Delta h_z(k-3))\\
=& \frac{1}{d}\left( -\abs{q_{z-1}^k-y_{z-1}} - \abs{q_{z-1}^{k-1} - y_{z-1}} + \abs{q_z^{k-2}-y_z} + \abs{q_z^{k-3}-y_z} \right)\\
=& 
\begin{cases}
\frac{1}{d}\left( \underbrace{y_{z-1-\llfloor\frac{k-1}{2}\rrfloor} - y_{z+1-\llfloor\frac{k-1}{2}\rrfloor}}_{\leq 0} + \underbrace{y_{z-2+\llceil\frac{k}{2}\rrceil}-y_{z-2+\llceil\frac{k}{2}\rrceil}}_{=0} \right) \leq 0, &k \text{ odd}\\
\frac{1}{d}\left( \underbrace{y_{z-\llfloor\frac{k}{2}\rrfloor} - y_{z+2-\llfloor\frac{k}{2}\rrfloor}}_{\leq 0} + \underbrace{y_{z-1+\llceil\frac{k-1}{2}\rrceil} - y_{z-1+\llceil\frac{k-1}{2}\rrceil}}_{=0} \right) \leq 0, &k \text{ even}
\end{cases}
\end{align*}
\smartqed
\end{proof}
Combining all of the above we can finally proof our main result as follows.
\begin{proof}[Theorem~\ref{thm:to7}]
We start by showing that every $k \in \setupto{m_{z+1}}$ with $k<{k^*_{z}-3}$ can not be an optimizer of $h_{z+1}$. It follows that $k^*_{z}-3 \leq k^*_{z+1}$, and, hence, $k^*_{z} \leq k^*_{z+1}+3$ as required for the upper bound. Indeed, we have
\begin{align*}
h_{z+1}(k) &= h_{z+1}(k+2) - (\Delta h_{z+1}(k+2)+\Delta h_{z+1}(k+1))\\
&\leq h_{z+1}(k+2) - (\Delta h_{z+1}(k^*_{z}-2) + \Delta h_{z+1}(k^*_{z}-3)) & \text{(by Eq.~\eqref{eq:smoothedconc})}\\
&\leq h_{z+1}(k+2) -  \underbrace{(\Delta h_{z}(k^*_{z}) + \Delta h_{z}(k^*_{z}-1))}_{>0 \text{ by def. of }k^*_z}  <h_{z+1}(k+2) \enspace . & \text{(by Lm.~\ref{lm:link})}
\end{align*}
Analogously, for the lower bound, we show that every $k \in \setupto{m_{z+1}}$ with $k>{k^*_{z}+3}$ can not be the smallest optimizer of $h_{z+1}$. It follows that $k^*_{z}+3 \geq k^*_{z+1}$, and, hence, $k^*_{z} \geq k^*_{z+1}-3$ as required. Indeed, we can write
\begin{align*}
h_{z+1}(k) &= h_{z+1}(k-2)+\Delta h_{z+1}(k)+\Delta h_{k+1}(k-1) \\
&\leq h_{z+1}(k-2) + \Delta h_{z+1}(k^*_z+4) + \Delta h_{z+1}(k^*_z+3)  & \text{(by Eq.~\eqref{eq:smoothedconc})}\\
&\leq h_{z+1}(k-2) + \underbrace{\Delta h_{z}(k^*_z+2) + \Delta h_{z}(k^*_z+1)}_{\leq 0 \text{ by def. of }k^*_z} \leq h_{z+1}(k-2) & \text{(by Lm.~\ref{lm:link})}
\end{align*}
\smartqed
\end{proof}

\section{Additional Proofs}
\label{apx:addproofs}
\begin{proof}[Prop.~\ref{prop:fastsumcomp}]
Using $d_{ij}$ as a shorthand for $y_j-y_i$ for $i,j \in \setupto{m}$ with $i\leq j$ we can write
\begin{align*}
&e_l(z)-e_l(a)-(a-1)(y_z-y_a)+e_r(z)-e_r(b)-(m-b)d_{zb}\\
=&\sum^{z-1}_{i=1}d_{iz}-\sum^{a-1}_{i=1}d_{ia}-(a-1)d_{az}+\sum_{i=z+1}^m d_{zi}-\sum_{i=b+1}^{m}d_{bi}-(m-b)d_{zb}\\
=&\sum^{z-1}_{i=a}d_{iz}+\sum^{a-1}_{i=1}\underbrace{(d_{iz}-d_{ia})}_{d_{az}}-(a-1)d_{az}+\sum^{b}_{i=z+1}d_{zi}+\sum^{m}_{i=b+1}\underbrace{(d_{zi}-d_{bi})}_{d_{zb}}-(m-b)d_{zb}\\
=&\sum^{z-1}_{i=a}d_{iz}+(a-1)d_{az}-(a-1)d_{az}+\sum^{b}_{i=z+1}d_{zi}+(m-b)d_{zb}-(m-b)d_{zb}\\
=&\sum^{z-1}_{i=a}d_{iz} + \sum^{b}_{i=z+1}d_{zi}=\sum_{i=a}^b \abs{y_z-y_i}
\end{align*}
\smartqed
\end{proof}

\newpage
\section{Summary of Used Notations}
\label{apx:notation}
\begin{center}
\begin{tabular}{lll}\toprule
Symbol & Meaning & Defined in\\
\midrule
$\card{\cdot}$ & cardinality of a set or absolute value of a number & - \\
$\setupto{k}$ & set of integers $\{1,\dots,k\}$ & \ref{sec:prelim}\\
$(x)_+$ & $\max \{x,0\}$ for a real-valued expression $x$ & \ref{sec:prelim}\\
$\domby$ & element-wise less-or-equal relation for multisets of real values & \ref{sec:tightoests}\\
$\powerset{X}$ & power set of a set $X$, i.e., set of all of its subsets & - \\
$\N^X$ & set of all multisets containing elements from set $X$ & - \\
$\sigma$, $\varphi$ & subgroup selectors $\map{\sigma, \varphi}{P}\{\true,\false\}$ & \ref{sec:langobjclosed}\\
$\tendsymb$ & measure of central tendency & \ref{sec:topseq}\\
$\dispsymb$ & measure of dispersion & \ref{sec:medianseq}\\
$e_l(i)$, $e_r(i)$ & left and right cumulative errors of target values up to value $i$ & \ref{sec:lineartime}\\
$f$ & objective function & \ref{sec:langobjclosed}\\
$\oest{f}$ & tight optimistic estimator of objective function $f$ & \ref{sec:bandb}\\
$m$ & number of elements in subpopulation $Q$ & \ref{sec:tightoests}\\
$m_z$ & maximal size parameter $k$ for consecutive value set $Q_{z}^k$ & \ref{sec:medianseq}\\
$k^*_z$ & $f$-maximizing size parameter $k$ for consecutive value set $Q_{z}^k$ & \ref{sec:medianseq}\\
$y$ & numeric target attribute $\map{y}{P}{\R}$ & \ref{sec:langobjclosed}\\
$y_i$ & $i$-th target value of subpopulation w.r.t. ascending order & \ref{sec:tightoests}\\
$P$ & global population of given subgroup discovery problem & \ref{sec:langobjclosed}\\
$Q$ & some subpopulation $Q \subseteq P$ & \ref{sec:langobjclosed}\\
$Q_z$ & median sequence element with median index $z$ & \ref{sec:medianseq}\\
$Q^k_z$ & submultiset of $Q$ with $k$ consecutive elements around index $z$ & \ref{sec:medianseq}\\
$T_i$ & top sequence element $i$, i.e., $T_i=\{y_{m-i+1},\dots, y_m\}$ & \ref{sec:topseq}\\
$Y$ & real-valued multiset & \ref{sec:tightoests}\\
$\diffsetmed{Y}$ & multiset of differences of elements in $Y$ to its median & \ref{sec:medianseq}\\
$\cL$, $\cnjlang$ & description language and language of conjunctions & \ref{sec:langobjclosed}\\
$\cnjclosedlang$ & language of closed conjunctions & \ref{sec:langobjclosed}\\
$\aamd{Q}$ & mean absolute deviation of $y$-values in $Q$ to their median & \ref{sec:medianseq}\\
$\cov{Q}$ & coverage, i.e., relative size $\card{Q}/\card{P}$ of subpopulation $Q$ & \ref{sec:langobjclosed}\\
$\mdcc{Q}$ & dispersion-corrected coverage of subpopulation $Q$ & \ref{sec:lineartime}\\
$\average{Q}$ & arithmetic mean of $y$-values in $Q$ & -\\
$\impact{Q}$ & impact, i.e., weighted mean-shift, of subpopulation $Q$ & \ref{sec:langobjclosed}\\
$\median{Q}$ & median of $y$-values in $Q$ & \ref{sec:topseq}\\
$\sumae{Q}$ & sum of absolute deviations of $y$-values in $Q$ to their median & \ref{sec:lineartime}\\
\bottomrule
\end{tabular}
\end{center}

\end{document}